\documentclass[numbers,webpdf,imanum]{ima-authoring-template}%
\usepackage{amsfonts}
\usepackage{amsmath}
\usepackage{amssymb}
\usepackage{subfig}
\usepackage{bm}
\usepackage{xcolor}
\usepackage{natbib}
\setcitestyle{authoryear,round,citesep={;\!}}
\graphicspath{{figs/}}
\catcode`\@=11
\numberwithin{equation}{section}

\newcommand{\Ba}{{\bm a}}

\newcommand{\Bw}{{\bm w}}
\newcommand{\Bx}{{\bm x}}
\newcommand{\Balpha}{{\bm \alpha}}
\newcommand{\Bzeta}{{\bm \zeta}}
\newcommand{\td}{{\rm d}}

\newcommand{\mE}{\mathbb{E}}
\newcommand{\tVar}{{\rm Var}}

\newtheorem{theorem}{Theorem}[section]
\newtheorem{lemma}[theorem]{Lemma}

\newtheorem{corollary}[theorem]{Corollary}

\newtheorem{remark}{Remark}[section]

\begin{document}
\DOI{}
\copyrightyear{}
\vol{}
\pubyear{}
\access{}
\appnotes{}
\copyrightstatement{}
\firstpage{1}

\title[optimization and generalization of PINN]{Optimization and generalization analysis for two-layer physics-informed neural networks without over-parametrization}

\author{Zhihan Zeng
\address{\orgdiv{School of Mathematical Sciences}, \orgname{University of Electronic Science and Technology of China}, \orgaddress{\state{Sichuan}, \country{China}}}}
\author{Yiqi Gu*
\address{\orgdiv{School of Mathematical Sciences}, \orgname{University of Electronic Science and Technology of China}, \orgaddress{\state{Sichuan}, \country{China}}}}
\authormark{Z. Zeng and Y. Gu}
\corresp[*]{Corresponding author: yiqigu@uestc.edu.cn}

\received{Date}{0}{Year}
\revised{Date}{0}{Year}
\accepted{Date}{0}{Year}


\abstract{This work focuses on the behavior of stochastic gradient descent (SGD) in solving least-squares regression with physics-informed neural networks (PINNs). Past work on this topic has been based on the over-parameterization regime, whose convergence may require the network width to increase vastly with the number of training samples. So, the theory derived from over-parameterization may incur prohibitive computational costs and is far from practical experiments. We perform new optimization and generalization analysis for SGD in training two-layer PINNs, making certain assumptions about the target function to avoid over-parameterization. Given $\epsilon>0$, we show that if the network width exceeds a threshold that depends only on $\epsilon$ and the problem, then the training loss and expected loss will decrease below $O(\epsilon)$.}

\allowdisplaybreaks
\sloppy
\maketitle

%

\section{Introduction}
Physics-informed neural networks (PINNs) have emerged as a promising approach to solving partial differential equations (PDEs) and other problems with physical constraints in recent years \citep{Raissi2019,Karniadakis2021}. Unlike traditional approximation functions, such as polynomials and finite elements, PINNs can alleviate the curse of dimensionality in some learning tasks, for example, whose target function is from the Barron space \citep{Barron1993, Ma2022_2, Wojtowytsch2022,Ma2022, Caragea2023}. This property makes PINNs particularly effective in high-dimensional problems \citep{Abbasi2024,Hu2024,Karniadakis2021,Cao2024}.

Despite numerous applications, the rigorous theoretical foundations of PINNs remain underdeveloped. Recent studies have explored the behaviors of gradient descent in training PINNs. However, these analyses are based on the hypothesis of over-parameterization, which means that the required number of neurons (i.e., the width of the network) grows polynomially with the number of training samples \footnote{In some other papers, over-parameterization includes the more general case where the number of neurons exceeds a certain threshold (may be independent of training samples). However, the concept in this paper does not include this case.}. Early work on learning theory reveals that the training loss of fully-connected neural networks (FNNs) will approach zero via gradient descent under the over-parameterization regime \citep{Soudry2016,Du2018,Du2019,Allen2019_2,Wu2019,Zou2019}. Some recent work reproduces the analysis for two-layer PINNs, proving that gradient descent can find global minima with zero training loss if the width is of $\Omega(N^p)$, where $N$ is the number of training samples and $p$ is some positive number \citep{Luo2020,Gao2023,Xu2024}. Despite enabling the success of gradient descent, over-parametrization incurs prohibitive computational costs and scalability limits. For example, in theory from \citep{Gao2023}, a standard PINN for heat equation requires $\Omega(N^2)$ neurons to achieve global minima; therefore, solving a problem with $10^4$ samples, the PINN needs at least $\Omega(10^8)$ neurons.

However, practical numerical experiments demonstrate that training loss can be reduced to low levels via gradient descent using much narrower neural networks than the theory suggests. An example lies in the work \citep{Grossmann2024} that gradient-based optimizers minimize the quadratic loss of two-layer PINNs for a 2-D Poisson equation; the loss evaluated at 2250 samples can be optimized to $O(10^{-4})$ (with PDE solution error being $O(10^{-2})$) using a narrow PINN with neurons merely 60.  

We believe the significant gap between the theoretically required and actual network width stems from the lack of assumptions about data labels. In previous work \citep{Luo2020,Gao2023}, PINNs are considered to fit the dataset $\{\left(x_n,f(x_n)\right)\}$, where $f$ is the governing function of the PDEs. Their results hold for general $f$, implying they are true even if the labels $f(x_n)$ are randomly given. However, real-world PDEs always have governing functions with special properties such as (piecewise) continuity or smoothness; therefore, the function $f$ learned by PINNs usually belongs to some special function class. In these cases, the labels are usually well distributed, and fitting them probably requires much fewer neurons. We expect the quantity to be independent of the number of training samples.

Some previous works have already studied the behavior of gradient descent in training FNNs, provided that the data is extracted from special functions. In \citep{Andoni2014}, the authors consider learning polynomials by two-layer FNNs, showing that gradient descent can decrease the quadratic loss below $\epsilon$ if the network width is $\Omega(1/\epsilon^3)$. In \citep{Allen2019}, a special class of functions is learned by two or three-layer ReLU FNNs via gradient descent, and the network width is required to be poly($1/\epsilon$) to decrease the training loss towards $\epsilon$. In these works \citep{Daniely2020,Barak2022,Jacot2018}, the required number of neurons only depends on the target function (including the input dimension) and is independent of the sample size. To the best of our knowledge, similar analyses for PINNs are still lacking.

\subsection{Our contributions}
In this paper, we investigate the behavior of stochastic gradient descent (SGD) in training two-layer PINNs. The results of the optimization and generalization are both developed. Specifically, we consider the PINN model for solving a $d$-dimensional Poisson's equation. The analysis is performed in three steps. 

Firstly, we formulate a function class $\mathcal{F}$ as well as its discretization $\mathcal{F}_m$. A universal approximation result is also developed between $\mathcal{F}$ and $\mathcal{F}_m$. Turning to the PINN model, we assume that the governing function $f$ of the PDE belongs to $\mathcal{F}$, and there exists a pseudo neural network $g\in\mathcal{F}_m$ that is close to $f$ up to any given accuracy $\epsilon$. We remark that the function class $\mathcal{F}$ is large, which contains at least all polynomials that vanish at zero.

Secondly, we perform the optimization analysis by estimating the difference between the PINN $\psi$ and the pseudo network $g$, as well as the gradient of their loss difference. Based on the estimation, we analyze the dynamics of SGD. The main theorem (Theorem \ref{thm3.4}) demonstrates that the average training loss is bounded above by $O(\epsilon)$, provided that the iteration number and learning rate are appropriately chosen, as long as $\psi$ is sufficiently wide. The width requirement only depends on $\epsilon$ and the PDE; namely, it is independent of the number of training samples. 

Finally, we derive generalization bounds for the average expected loss using Rademacher complexity. In the main theorem (Theorem \ref{thm4.2}), we prove that the average expected loss is no more than $O(\epsilon)$ by further assuming that there are sufficiently many training samples. Although we conduct the analysis only for Poisson's equation in this paper, the discussion can be generalized to other types of PDEs. 

\subsection{Organization of the paper}
This paper is organized as follows. In section 2, we review the PINN-based least squares method and the practical SGD algorithms. In Section 3, we define the conceptual class of the target function and discuss its finite-parametrized approximation. Moreover, we prove that SGD can decrease training loss to low levels. In Section 4, we prove the same bound for the generalized loss using Rademacher complexity. A numerical example is presented in Section 5 to validate the preceding theory. Conclusions and discussions about further research work are provided in Section 6.

\section{Preliminaries}
\subsection{Notations} 
We let $\mathcal{U}(-a,a)$ be the uniform distribution in the interval $[-a,a]$, and let $\mathbb{I}_E$ be the characteristic function of a region $E$. For $a,b\geq0$, we use the notation $a=O(b)$, or equivalently $b=\Omega(a)$, if there exists a constant $C>0$ independent of $a$ and $b$ such that $a\leq Cb$. Similarly, we use $a=\Theta(b)$ to mean that there exist two constants $C_1,C_2 > 0$ such that $C_1b \leq a\leq C_2b$. For any positive integer $n$, we denote $[n]=\{1,2,\dots,n\}$.

We use $\|\cdot\|_1$ and $\|\cdot\|_2$ to denote the $1$-norm and Euclidean norm of a column/row vector or a vector-valued function, respectively. Also, we define the matrix norm $\|\cdot\|_{2,p}$ with $p \geq 1 $ by 
     \begin{equation*}
         \|W\|_{2,p} := \left(\sum_{i=1}^m\|\Bw_i\|_2^p\right)^{1/p},\quad\forall~W\in \mathbb{R}^{m\times n},
     \end{equation*}
where $\Bw_i$ is the $i$-th row of $W$.

\subsection{Problem and PINN model}\label{sec_Problem_and_PINN_model}
In this paper, we take Poisson's equation on a unit ball as an example to show the analysis. Similar arguments can be applied to other types of PDEs on domains of different shapes. Let $\Gamma=\{\Bx\in\mathbb{R}^d: \|\Bx\|_2\leq 1 \}$ be the $d$-dimensional unit ball, then the Poisson's equation with homogeneous Dirichlet boundary condition is given by
\begin{equation} \label{01}
\begin{cases}
    \Delta u(\Bx)=f(\Bx),  &\text{in}~\Gamma, \\
    u(\Bx)=0,  &\text{on}~\partial\Gamma.
\end{cases} 
\end{equation}
Here, $f$ is a given function, and $u$ is the unknown solution. Throughout this paper, we regard the dimension $d$ as a fixed number, which can be absorbed in the constants of $O(\cdot)$, $\Omega(\cdot)$ and $\Theta(\cdot)$ since $d$ depends only on the problem. For consistency of analysis (see Section \ref{sec_Classes_of_functions}), we only consider the case that $f(0)=0$. Otherwise, we can let $v(\Bx)=u(\Bx)+\frac{f(0)}{2d}\|\Bx\|_2^2(\|\Bx\|_2^2-1)$, then $v$ satisfies the equation
\begin{equation}\label{09}
\Delta v(\Bx)=f(\Bx)+f(0)\left[(2+\frac{4}{d})\|\Bx\|_2^2-1\right]
\end{equation}
in $\Gamma$, where the right-hand side vanishes at $\Bx=0$, and $v$ preserves the homogeneous Dirichlet condition on $\partial\Gamma$. It suffices to solve \eqref{09} for $v$, and $u$ can be obtained immediately by $u(\Bx)=v(\Bx)-\frac{f(0)}{2d}\|\Bx\|_2^2(\|\Bx\|_2^2-1)$. 

One approach to solving \eqref{01} is to use a neural network to approximate the solution $u$.
Specifically, writing $\Bx=[x_1~\dots~x_d]^\top$ in the column vector form, one can take the function
\begin{equation}\label{02}
   \phi(\Bx)= \left( \|\Bx\|_2^2 - 1 \right) \tilde{\phi}(\Bx)
\end{equation}
as the approximate solution, where 
\begin{equation}\label{05}
    \tilde{\phi}(\Bx)=\sum_{i=1}^m a_i \sigma(\Bw_i^\top \Bx + b_i),
\end{equation}
is a two-layer FNN. Here, $m>0$ is the width of the network; $\sigma(\cdot)$ is the activation function; $a_i\in\mathbb{R}$ is the weight of the output layer; $\Bw_i\in\mathbb{R}^d$ and $b_i\in\mathbb{R}$ are the (column) weight vector and bias scalar in the hidden layer, respectively. In this paper, we consider the case that  $\sigma(\cdot)$ is the $\text{ReLU}^3$ activation function, i.e., $\sigma(t)=\max{(0,t^3)}$, which is frequently used to solve second-order PDEs.  

Note that the approximate solution $\phi(\Bx)$ defined in \eqref{02} always satisfies the boundary condition $\phi(\Bx)=0$ on $\partial\Gamma$. So, it suffices to fulfill the differential equation $\Delta\phi=f$ in $\Gamma$. A common strategy is minimizing the $L^2$ residual, namely,
\begin{equation}\label{03}
\min_{\psi}\|\psi-f\|_{L^2(\Gamma)}^2,
\end{equation}
where $\psi:=\Delta\phi$ is the PINN associated with the PDE \eqref{01}.

The minimization \eqref{03} formulates a least squares regression problem: given a target function $f$, it is expected to find a good learner network $\psi$ so that the $L^2$ error is small. In practice, the $L^2$ norm in \eqref{03} is computed in the discrete sense. Specifically, we generate a set of training points $X:=\{\Bx_n\}_{n=1}^N\subset \Gamma$, which are i.i.d random variables under some distribution $\mathcal{D}$. Then $\{(\Bx_n,f(\Bx_n)\}_{n=1}^N$ forms a dataset, and the PINN model \eqref{03} becomes
\begin{equation}\label{04}
    \min_{\psi} \frac{1}{N}\sum_{n=1}^N|\psi(\Bx_n)-f(\Bx_n)|^2.
\end{equation} 

Note that the learner network $\psi(\Bx)$ has the expression
\begin{equation}\label{06}
\begin{split}
    \psi(\Bx)=&~\Delta \phi(\Bx)=\Delta\left[\left( \|\Bx\|_2^2 - 1 \right) \sum_{i=1}^m a_i \sigma(\Bw_i^\top\Bx + b_i)\right]
    =2d\sum_{i=1}^{m} a_i (\Bw_i^\top\Bx+b_i )^3\mathbb{I}_{ \Bw_i^\top \Bx + b_i \geq 0 }\\
   &+12\sum_{i=1}^{m} a_i (\Bw_i^\top \Bx + b_i )^2(\Bw_i^\top \Bx )\mathbb{I}_{ \Bw_i^\top \Bx + b_i \geq 0 }+6\sum_{i=1}^{m} a_i(\Bw_i^\top \Bx + b_i)(\Bw_i^\top \Bw_i)(\|\Bx\|_2^2-1)  \mathbb{I}_{ \Bw_i^\top \Bx + b_i \geq 0 }, \\
\end{split}
\end{equation}
which is determined by the parameters $\{a_i,\Bw_i,b_i\}$. So, our goal is to minimize the loss function in \eqref{04} by tuning these parameters commonly implemented by SGD or its variants. In this paper, we set a target accuracy $\epsilon>0$ and discuss under which situation the loss function will be decreased below $O(\epsilon)$.

\subsection{Stochastic gradient descent}\label{sec_SGD}
Now, we consider using SGD to solve \eqref{04}. Firstly, we initialize $\psi$ by assigning
\begin{equation}\label{07}
 a_i \leftarrow a_i^{(0)}\sim\mathcal{U}(- m^{-\alpha}, m^{-\alpha}), \quad \Bw_i \leftarrow \Bw_i^{(0)}\sim\mathcal{U}(-m^{-\beta},m^{-\beta}), \quad b_i \leftarrow b_i^{(0)}\sim\mathcal{U}(-m^{-\beta},m^{-\beta}),
\end{equation}
where $\alpha,\beta \in[0,\infty)$ are some powers. In previous works studying FNNs (e.g., \cite{Du2018,Allen2019}), $(\alpha,\beta)$ are typically set to $(0,\frac{1}{2})$, which ensures the stability of parameters during backward propagation. However, PINNs have slightly different propagation schemes from FNNs. So, here, we use general powers for discussion instead of specific values.

For simplicity, we fix $a_i$ and $b_i$ once they have been initialized and only tune the weight vectors $\Bw_i$ in minimization \eqref{04}. We let $W:=[\Bw_1~\dots~\Bw_m]$ be the matrix with columns being the trainable weight vectors and rewrite $\psi(\Bx)=\psi(\Bx;W)$. Then the minimization \eqref{04} can be reformulated as 
\begin{equation}\label{48}
    \min_W\mathcal{L}_X(\psi(\Bx;W)):=\frac{1}{N}\sum_{n=1}^N\mathcal{L}(\psi(\Bx_n;W)),
\end{equation}
where $\mathcal{L}(\psi(\Bx;W)):=|\psi(\Bx;W)-f(\Bx)|^2$.

We use $W^{(t)}$ to denote the weight $W$ after $t$ iterations of gradient descent, and let $W_t:=W^{(t)}-W^{(0)}$. Then, the SGD algorithm is given by
\begin{equation*}
    \begin{split}
       \text{for}~ &t=1,2,\dots,T\\
       &\Bx\sim \mathcal{U}(X)\\
      &W_t \leftarrow W_{t-1}-\eta\nabla_W\mathcal{L}(\psi(\Bx;W_{t-1}+W^{(0)}))\\
   \end{split}
\end{equation*}    
where $\Bx\sim \mathcal{U}(X)$ means that we randomly select one point $\Bx$ from $X$ with uniform distribution; $T$ is the total number of iterations; $\eta>0$ is the learning rate. Therefore, the final result of the PINN model will be affected by three random factors: the random initialization $\{a_i^{(0)},\Bw_i^{(0)},b_i^{(0)}\}$, the random dataset $X$, and the random selection $\Bx\sim \mathcal{U}(X)$ in every SGD iteration.

\subsection{Classes of functions}\label{sec_Classes_of_functions}
Let $\theta:=\left(a^{(0)},\Bw^{(0)},b^{(0)}\right)$ be the vector consisting of random variables that obey the distribution given in \eqref{07}. So $\theta$ is a random variable with uniform distribution in the region 
\begin{equation}
\Lambda:=[- m^{-\alpha}, m^{-\alpha}]\times[-m^{-\beta},m^{-\beta}]^d\times[-m^{-\beta},m^{-\beta}]
\subset\mathbb{R}\times\mathbb{R}^d\times\mathbb{R},
\end{equation} 
whose density function is $p(\theta)=\frac{1}{|\Lambda|}\mathbb{I}_{\theta\in\Lambda}$. Next, we define the random basis associated with PINN by
\begin{multline}\label{08}
\Bzeta(\Bx;\theta)=a^{(0)}\Big(2d({\Bw^{(0)}}^\top\Bx+b^{(0)})^2+12({\Bw^{(0)}}^\top\Bx)({\Bw^{(0)}}^\top\Bx+b^{(0)})\\
    +6({\Bw^{(0)}}^\top\Bw^{(0)})(\|\Bx\|_2^2-1)\Big)\Bx\cdot\mathbb{I}_{{\Bw^{(0)}}^\top\Bx+b^{(0)}\geq 0}.
\end{multline}
And we let $\mathcal{F}$ be the function class consisting of all functions that can be written as an infinite linear combination of $\Bzeta(\Bx;\theta)$ over the parameter $\theta$; namely,
\begin{multline}\label{47}
\mathcal{F}:=\Bigg\{f:\Gamma \rightarrow \mathbb{R},~f(\Bx)=\int_{\Lambda} \Balpha(\theta)^\top\Bzeta(\Bx;\theta)\td \theta\\
\text{for some vector-valued function}~\Balpha(\theta): \Lambda\rightarrow\mathbb{R}^d\Bigg\}.
\end{multline}
Since $\Bzeta(0,\theta)=0$, we have $f(0)=0$ for all $f\in\mathcal{F}$. In our theory, the right-hand side function in the PDE \eqref{01} is required to be in $\mathcal{F}$. So, we assumed $f(0)=0$ in the PDE in Section \ref{sec_Problem_and_PINN_model} for consistency. Also, we equip $\mathcal{F}$ with the norm 
\begin{equation}\label{16}
\|f\|_{\mathcal{F}} := \inf_{\Balpha} \max_{\theta\in \Lambda} \frac{\|\Balpha(\theta)\|_2}{p(\theta)}=|\Lambda|\inf_{\Balpha} \max_{\theta\in \Lambda}  \|\Balpha(\theta)\|_2,
\end{equation}
where the infimum is taken over all possible functions $\Balpha(\theta)$ such that $f(\Bx)=\int_{\Lambda} \Balpha(\theta)^\top\Bzeta(\Bx;\theta)\td \theta$ holds.
\begin{remark}
The function space $\mathcal{F}$ is not very special and contains many common types of functions. For example, in the case of $d=1$, we take $\Balpha(\theta)=\tilde{\Balpha}(a^{(0)},\Bw^{(0)})(b^{(0)})^\gamma$, where $\tilde{\Balpha}\in L^1([-m^{-\alpha},m^{-\alpha}]\times[-m^{-\beta},m^{-\beta}])$ and $\gamma\in\mathbb{N}$, in \eqref{47}. By simple calculation on the multiple integrals, we obtain that
\begin{equation}
f(\Bx)=C_1\Bx^{\gamma+4}+C_2\Bx^{\gamma+3}+C_3\Bx^{\gamma+2}+C_4\Bx^3+C_5\Bx,
\end{equation}
where $C_i~(i=1,\dots,5)$ are coefficients only depending on $m,\alpha,\beta$ and the function $\tilde{\Balpha}$. Therefore, denoting $\mathbb{P}$ as the class of polynomials, if we take $\Balpha(\theta)=\tilde{\Balpha}(a^{(0)},\Bw^{(0)})q(b^{(0)})$ for all $q\in\mathbb{P}$, then $f(\Bx)$ ranges over $\Bx\mathbb{P}[\Bx]$, namely $\{p\in\mathbb{P}: p(0)=0\}$. So, $\mathcal{F}$ contains all polynomials that vanish at zero.
\end{remark}
Similarly, we define a function class, which can be seen as the discretization of $\mathcal{F}$, i.e.,
\begin{equation}
\mathcal{F}_{m} :=\left\{ g:\Gamma \rightarrow \mathbb{R},~ g(\Bx)= \sum_{i=1}^{m} \bm{\alpha}_i^\top \Bzeta(\Bx; \theta_i)~\text{for some}~\Balpha_i\in\mathbb{R}^d\right\},
\end{equation}
where $\theta_i$ are independent and identically distributed (i.i.d.) random variables with $\theta$. It is intuitive to see that the functions in $\mathcal{F}_{m}$ can approximate those in $\mathcal{F}$ as $m\rightarrow\infty$. We can prove this approximation in the $L^2$ sense. For this purpose, we first introduce the following inequalities.

\begin{lemma}\label{lem01} [Jensen’s inequality]
Suppose $\nu(\cdot)$ is a convex function and $\xi$ is a random variable. Then it holds that
\begin{equation} \mE(\nu(\xi)) \geq \nu(\mE(\xi)).\end{equation}
\end{lemma}

\begin{lemma}\label{lem02} [McDiarmid’s inequality]
Let $h:D_1\times D_2\times\dots\times D_n\rightarrow \mathbb{R}$. If for all $i=1,\dots,n$,  it holds that
\begin{equation} \left|h(t_1,\dots,t_i,\dots,t_n)-h(t_1,\dots,t_i',\dots,t_n)\right|\leq c_i,\end{equation}
for all $t_1\in D_1,\dots,t_n\in D_n$ and $t_i'\in D_i$, where $c_i>0$ is a constant. Then for every $\epsilon>0$, we have
\begin{equation}\mathbb{P}\{h(\xi_1,\dots,\xi_n)-\mE[h(\xi_1,\dots,\xi_n)]\geq \epsilon \}\leq \exp\left( \frac{-2\epsilon^2}{\sum_{i=1}^n c_i^2}\right),\end{equation}
where $\xi_1,\dots,\xi_n$ are i.i.d. random variables in $D_1,\cdots,D_n$, respectively.
\end{lemma}

Next, we estimate the error between the mean of bounded i.i.d. random variables and their expectation. The proof of the following two lemmas is in Appendix A.
\begin{lemma}\label{lem03}
Let $\Xi = \{\xi_1, \cdots, \xi_m\}$ be random variables i.i.d. satisfying $\|\xi_i\|\leq C $  for $i=1,\dots,m$ in a Hilbert space $\mathcal{H}$, where $\|\cdot\|$ means the norm relate to the space $\mathcal{H}$ and $C$ is a constant. Denote their average by $\overline{\Xi} = \frac{1}{m} \sum_{i=1}^{m} \xi_i$. Then for any $\delta>0$, with probability at least $1-\delta$ we have
\begin{equation}
\|\overline{\Xi}-\mE\overline{\Xi}\| \leq \frac{C}{\sqrt{m}} \left( 1 + \sqrt{2 \log \frac{1}{\delta}} \right).
\end{equation}
\end{lemma}

Finally, given $f\in\mathcal{F}$, we can estimate the best $L^2$ approximation by $\mathcal{F}_m$. Note that $\mathcal{F}_m$ is determined by the random variables $\theta_1, \ldots, \theta_m$. 
\begin{lemma} \label{lem04}
Suppose that $\mu$ is any probability measure on $\Gamma$ and $f\in\mathcal{F}$. Let $m\in\mathbb{N}^+$, then for any $\delta > 0$, with probability at least $1-\delta$ over $\theta_1, \ldots, \theta_m$, there exists a function $g\in\mathcal{F}_{m}$ with $\|\Balpha_i\|_2\leq\frac{\|f\|_{\mathcal{F}}}{m}$ such that
\begin{equation} \label{26}
\sqrt{\int_{\Gamma} \left( g(\Bx) - f(\Bx) \right)^2 \td\mu(\Bx)} \leq C_d\|f\|_\mathcal{F}m^{-\alpha-2\beta-1/2} \left( 1 + \sqrt{2 \log \frac{1}{\delta}} \right),
\end{equation}
where $C_d:=2 d^{5/2}+4d^2+26d^{3/2}+12d$.
\end{lemma}
\subsection{Rademacher complexity}
Rademacher complexity serves as a foundational framework for studying generalization bounds. Here we list several useful results that can be found in the literature on machine learning (e.g., \cite{Shalev2014}) 

Let $\mathcal{H}$ be a class of functions from $\mathbb{R}^d$ to $\mathbb{R}$ and $X = (\bm{x}_1, \dots, \bm{x}_N)$ be a finite set of samples in $\mathbb{R}^d$. Then the empirical Rademacher complexity with respect to $X$ of $\mathcal{H}$ is defined by
\begin{equation}
\widehat{\mathcal{R}}(X; \mathcal{H}) := \mE_{\xi \sim \{\pm 1\}^N} \left[ \sup_{h \in \mathcal{H}} \frac{1}{N} \sum_{n=1}^N \xi_n h(\bm{x}_n) \right],
\end{equation}
where $\xi=(\xi_1,\dots,\xi_N)$ are random variables of binary uniform distribution. i.e., $\mathbb{P}(\xi_n=1)=\mathbb{P}(\xi_n=-1)=\frac{1}{2}$.
\begin{lemma} \label{lem 05}[Basic properties of Rademacher complexity]  Let $\sigma : \mathbb{R} \rightarrow \mathbb{R}$ be a fixed 1-Lipschitz function.
\begin{enumerate}
    \item[(a)]~Suppose $\|\Bx\|_2 \leq 1$ for all $\Bx \in X$. The class $\mathcal{H} = \{\Bx \mapsto \Bw^\top\Bx \mid \|\Bw\|_2 \leq C\}$ has Rademacher complexity $\widehat{\mathcal{R}}(X; \mathcal{H}) \leq O\left(\frac{C}{\sqrt{N}}\right)$;
    
    \item[(b)] ~$\widehat{\mathcal{R}}(X; \mathcal{H}_1 + \mathcal{H}_2) = \widehat{\mathcal{R}}(X; \mathcal{H}_1) + \widehat{\mathcal{R}}(X; \mathcal{H}_2)$;
    
    \item[(c)]~ Let $\mathcal{H}_1, \dots, \mathcal{H}_{m}$ be $m$ classes of functions and $\Bw=[w_1,\dots,w_m]\in \mathbb{R}^m$ be a fixed vector, then $\mathcal{H}' = \left\{ \Bx \mapsto \sum_{j=1}^m w_j \sigma(h_j(\Bx)) \mid h_j \in \mathcal{H}_j \right\}$ satisfies $\widehat{\mathcal{R}}(X; \mathcal{H}') \leq 2 \|\Bw\|_1 \max_{j \in [m]} \widehat{\mathcal{R}}(X; \mathcal{H}_j)$;
\end{enumerate}
\end{lemma}

\begin{lemma} \label{lem06}[Rademacher generalization]
Suppose $X = (\bm{x}_1, \dots, \bm{x}_N)$ with each $\Bx_i$ being generated i.i.d. from a distribution $\mathcal{D}$. Let $\mathcal{H}$ be a set of functions satisfying $|h|\leq C~\forall h\in \mathcal{H}$. Then for every $\delta \in (0,1)$, with probability at least $1-\delta$ over the randomness of $X$, it satisfies
\begin{equation}
\sup_{h \in \mathcal{H}} \left| \mE_{\Bx \sim \mathcal{D}}[h(\bm{x})] - \frac{1}{N} \sum_{n=1}^N h(\bm{x}_n) \right| \leq 2 \widehat{\mathcal{R}}(X; \mathcal{H}) + O\left(\frac{C \sqrt{\log(1/\delta)}}{\sqrt{N}}\right).
\end{equation}
\end{lemma}

Moreover, one can prove the following result using the contraction lemma for the Rademacher complexity.
\begin{corollary}[\cite{Allen2019}] \label{cor02}
Suppose $X = (\bm{x}_1, \dots, \bm{x}_N)$ with each $\Bx_i$ being generated i.i.d. from a distribution $\mathcal{D}$. Let $\mathcal{H}$ be a class of functions and $\ell:\mathbb{R}\rightarrow[-C,C]$ be a $C_L$-Lipschitz continuous function. Then
 \begin{equation}
     \sup_{h\in \mathcal{H}}\left|\mathbb{E}_{\Bx\sim \mathcal{D}}[\ell(h(\Bx))]-\frac{1}{N}\sum_{n=1}^N \ell(h(\Bx_n))\right| \leq 2 C_L\widehat{\mathcal{R}}(X;\mathcal{H}) + O\left(\frac{C\sqrt{\log{(1/\delta)}}}{\sqrt{N}}\right).
\end{equation}
\end{corollary}
\section{Optimization Analysis} 
Our analysis begins by demonstrating that, under random initialization, a pseudo network exists in the vicinity of the initialization that can approximate the target function (Theorem \ref{thm03} and Corollary \ref{cor01}).  We then proceed to show that, in the neighborhood of the initialization, the PINN trained by SGD is close to the pseudo network in some sense (Theorem \ref{thm3.3}). By the connection of the pseudo network, we prove that the trained PINN can approximate the target function, leading to a small average training loss (Theorem \ref{thm3.4}). 

First, we assume that the SGD algorithm does not explode in the following sense
\begin{equation}\label{15}
\Bw_i^{(t)} \leq O(1),\quad \psi(\Bx;W^{(t)})\leq O(1),\quad\text{for}~t=1,\dots,T.
\end{equation}
 This assumption means that the parameter $\Bw_i$ and the PINN $\psi$ are always bounded above during iterations of SGD. If not, $\Bw_i$  or $\psi$ will blow up to infinity, causing the exploding gradient and the failure of gradient descent. In practical implementation, we always tune the hyperparameters to prevent the gradient from exploding, ensuring that $\psi$ remains bounded. However, at present, we cannot provide a theoretical guarantee that the above assumption is valid.

\subsection{Approximation} We will prove that any function in $\mathcal{F}$ can be closely approximated by functions in $\mathcal{F}_m$. We use the following norm notations for a matrix $W=[\Bw_1~\dots~\Bw_m]$:
\begin{equation*}
\|W\|_{2,\infty}:=\max_{1\leq i\leq m}\|\Bw_i\|_2,\quad\|W\|_F:=\left(\sum_{i=1}^m\|\Bw_i\|_2^2\right)^{1/2}.
\end{equation*}

Next, define the following parametrized function
\begin{equation}
    \begin{split}
        g^{(b)}(\bm{x};W)&=2d\sum_{i=1}^ma_i^{(0)}(\Bw_i^\top \Bx)({\Bw_i^{(0)}}^\top\Bx+b_i^{(0)})^2\mathbb{I}_{{\Bw_i^{(0)}}^\top\Bx+b_i^{(0)}\geq 0}\\
        &+12\sum_{i=1}^ma_i^{(0)}(\Bw_i^\top \Bx)({\Bw_i^{(0)}}^\top\Bx)({\Bw_i^{(0)}}^\top\Bx+b_i^{(0)})\mathbb{I}_{{\Bw_i^{(0)}}^\top\Bx+b_i^{(0)}\geq 0}\\
        &+6\sum_{i=1}^ma_i^{(0)}(\Bw_i^\top \Bx)({\Bw_i^{(0)}}^\top\Bw_i^{(0)})(\|\Bx\|^2_2-1)\mathbb{I}_{{\Bw_i^{(0)}}^\top\Bx+b_i^{(0)}\geq 0},  \label{32}     
    \end{split}
\end{equation}
where $a_i^{(0)}, \Bw_i^{(0)}, b_i^{(0)}$ are random variables with distribution \eqref{07}.

\begin{theorem}\label{thm03}
Suppose $f\in \mathcal{F}$ and $\mu $ is a probability measure with respect to a probability distribution $\mathcal{D}$. Given $\epsilon\in(0,1]$ and $\delta> 0 $, we let $M\geq \left((2C_d \|f\|_{{\mathcal{F}}}(1+\sqrt{2\log\frac{1}{\delta}}))/\epsilon\right)^{1/(\alpha+2\beta+\frac{1}{2})}$
with $C_d$ defined in Lemma \ref{lem04}. Then for any $m\geq M$, with probability at least $1-\delta$ over the random initialization  $a_i^{(0)}, \Bw_i^{(0)}, b_i^{(0)}$, there exists $W^*=[\Bw_1^*~\dots~\Bw_m^*]$ with $\|W^*\|_{2,\infty}\leq \frac{\|f\|_{\mathcal{F}}}{m}$ and $\|W^*\|_F\leq \frac{\|f\|_{\mathcal{F}}}{\sqrt{m}}$ such that 
\begin{equation}\label{14}
\int_{{\Gamma}} \left( f(\bm{x})-g^{(b)}(\bm{x};W^*)\right)^2 \td\mu(\Bx) \leq \frac{\epsilon^2}{4};
\end{equation}
namely,
\begin{equation}
\mE_{\bm{x} \sim \mathcal{D}} \left[\left|g^{(b)}(\bm{x};W^*)-f(\Bx)\right|^2\right]\leq \frac{\epsilon^2}{4}.
\end{equation}
\end{theorem}
\begin{proof}
By Lemma \ref{lem04}, with probability at least $1-\delta$ over $\theta_1, \ldots, \theta_m$, there exists a function in $\mathcal{F}_{m}$, expressed by $\sum_{i=1}^m{\Bw_i^*}^\top\Bzeta(\bm{x};\theta_i)$ such that 
\begin{equation}\label{12}
\sqrt{\int_{{\Gamma}} \left( f(\bm{x})-\sum_{i=1}^m{\Bw_i^*}^\top\Bzeta(\bm{x};\theta_i) \right)^2 \td\mu(\bm{x})} \leq C_d\|f\|_\mathcal{F}m^{-\alpha-2\beta-1/2}\left( 1 + \sqrt{2 \log \frac{1}{\delta}} \right),
\end{equation}
with $\|\Bw_i^*\|_2\leq\frac{\|f\|_{\mathcal{F}}}{m}$. Therefore, $\|W^*\|_{2,\infty}\leq \frac{\|f\|_{\mathcal{F}}}{m}$ and $\|W^*\|_F\leq \frac{\|f\|_{\mathcal{F}}}{\sqrt{m}}$. If $m\geq M$, the right hand side of \eqref{12} is less than $\epsilon^{1/2}$. Then the proof is completed by the fact that $g^{(b)}(\bm{x};W^*)=\sum_{i=1}^m{\Bw_i^*}^\top\Bzeta(\bm{x};\theta_i).$
\end{proof}

Next, we define a pseudo network $g$ that can be seen as the linearization of $\psi$ formulated by \eqref{06}:
\begin{multline}\label{18}
g(\bm{x};W)=2d\sum_{i=1}^ma_i^{(0)}(\Bw_i^\top \Bx+b_i^{(0)})({\Bw_i^{(0)}}^\top\Bx+b_i^{(0)})^2\mathbb{I}_{{\Bw_i^{(0)}}^\top\Bx+b_i^{(0)}\geq 0}\\
+12\sum_{i=1}^ma_i^{(0)}(\Bw_i^\top \Bx+b_i^{(0)})({\Bw_i^{(0)}}^\top\Bx)({\Bw_i^{(0)}}^\top\Bx+b_i^{(0)})\mathbb{I}_{{\Bw_i^{(0)}}^\top\Bx+b_i^{(0)}\geq 0}\\
+6\sum_{i=1}^ma_i^{(0)}(\Bw_i^\top \Bx+b_i^{(0)})({\Bw_i^{(0)}}^\top\Bw_i^{(0)})(\|\Bx\|^2_2-1)\mathbb{I}_{{\Bw_i^{(0)}}^\top\Bx+b_i^{(0)}\geq 0}.
\end{multline}
Note that if we remove the bias $b_i^{(0)}$ from the term $\Bw_i^\top \Bx+b_i^{(0)}$, then $g(\bm{x};W)$ changes to $g^{(b)}(\bm{x};W)$. We can prove that the approximation property of $g^{(b)}(\bm{x};W)$ given by Theorem \ref{thm03} also holds for $g(\bm{x};W)$.

\begin{corollary}\label{cor01} 
Under the hypothesis of Theorem \ref{thm03}, we further assume that $M\geq (\frac{C_d'}{\epsilon})^{1/(\alpha+3\beta-1)}$ with $C_d':=4d^{5/2}+12d^2+60d^{3/2}+76d+24d^{1/2}$. Suppose $\alpha$ and $\beta$ satisfy $\alpha+3\beta> 1$. Then for any $m\geq M$, with probability at least $1-\delta$ over the random initialization $a_i^{(0)}$, $ \Bw_i^{(0)}$, $b_i^{(0)}$, there exists $W^*=[\Bw_1^*~\dots~\Bw_m^*]$ with $\|W^*\|_{2,\infty}\leq \frac{\|f\|_{\mathcal{F}}}{m}$ and $\|W^*\|_F\leq \frac{\|f\|_{\mathcal{F}}}{\sqrt{m}}$ such that
\begin{equation}
\mE_{\bm{x} \sim \mathcal{D}} \left[\left|f(\Bx)-g(\bm{x};W^{(0)}+W^*)\right|^2\right]\leq  \epsilon.
\end{equation}
\end{corollary}
\begin{proof}
By Theorem \ref{thm03}, there exists $W^*=[\Bw_1^*~\dots~\Bw_m^*]$ with $\|W^*\|_{2,\infty}\leq \frac{\|f\|_{\mathcal{F}}}{m}$, $\|W^*\|_F\leq \frac{\|f\|_{\mathcal{F}}}{\sqrt{m}}$ such that \eqref{14} holds. Then we have
\begin{multline*} 
\left|g(\Bx;W^{(0)}+W^*)-g^{(b)}(\Bx;W^*)\right|
=\Big|2d\sum_{i=1}^m a_i^{(0)}({\Bw_i^{(0)}}^\top\Bx+b_i^{(0)})^3\mathbb{I}_{{\Bw_i^{(0)}}^\top\Bx+b_i^{(0)}\geq 0}\\
+12\sum_{i=1}^ma_i^{(0)}({\Bw_i^{(0)}}^\top\Bx+b_i^{(0)})^2({\Bw_i^{(0)}}^\top\Bx)\mathbb{I}_{{\Bw_i^{(0)}}^\top\Bx+b_i^{(0)}\geq 0}\\
+6\sum_{i=1}^ma_i^{(0)}({\Bw_i^{(0)}}^\top\Bx+b_i^{(0)})({\Bw_i^{(0)}}^\top\Bw_i^{(0)})(\|\Bx\|^2_2-1)\mathbb{I}_{{\Bw_i^{(0)}}^\top\Bx+b_i^{(0)}\geq 0} \Big|\\ 
\leq m\cdot m^{-\alpha}\cdot\Bigg[2d \left(d^{1/2}m^{-\beta}+m^{-\beta}\right)^3+12d^{1/2}m^{-\beta}\left(d^{1/2}m^{-\beta}+m^{-\beta}\right)^2+ 12 d m^{-2\beta}\left(d^{1/2}m^{-\beta}+m^{-\beta}\right)\Bigg]\\
\leq m ^{1-\alpha-3\beta} \cdot C_d'\leq  \frac{\epsilon}{2} \end{multline*}
since $m \geq M$. Next, we have the following inequality in probability measure space $\Gamma$,
\begin{equation} \label{eq2}
\left(\int_{{\Gamma}}\left|g(\Bx;W^{(0)}+W^*)-g^{(b)}(\Bx;W^*)\right|^2\td\mu(\Bx)\right)^{1/2} \leq \frac{\epsilon}{2}.
\end{equation}
Clearly, using \eqref{eq2}, \eqref{14} and the triangle inequality, it follows that
\begin{multline*}
\left(\int_{{\Gamma}}\left|f(\Bx)-g(\Bx;W^{(0)}+W^*)\right|^2\td\mu(\Bx)\right)^{1/2} 
\leq \left( \int_{\Gamma}\left|f(\Bx) -g^{(b)} (\Bx;W^*) \right|^2 \td\mu(\Bx)\right)^{1/2}\\+\left(\int_{\Gamma} \left|g^{(b)}(\Bx;W^*)-g(\Bx;W^{(0)}+W^*)\right|^2 \td \mu(\Bx)\right)^{1/2}
\leq \epsilon,
\end{multline*}
thus, we arrive at the conclusion that
\begin{equation*}
\mE_{\bm{x} \sim \mathcal{D}} \left[\left|f(\Bx)-g(\bm{x};W^{(0)}+W^*)\right|^2\right]=\int_{\Gamma} \left(f(\Bx)-g(\Bx;W^{(0)}+W^*)\right)^2 \td\mu(\Bx)
\leq \epsilon.   
\end{equation*}
\end{proof}
\subsection{Distance between the learner and pseudo network}
The pseudo network $g$ serves as a connection between the learner network $\psi$ and the target function $f$. Here, we estimate the distance between $\psi$ and $g$ in the following sense.
\begin{theorem} \label{thm3.3}
Suppose $f\in\mathcal{F}$.  Then under the random initialization \eqref{07}, for every $\Bx\in\Gamma$ and every $t\in[T]$, it holds that
\begin{align*}
&\text{(a)} \quad\|\Bw^{(t)}_i - \Bw^{(0)}_i \|_2 \leq O(\eta t m^{-\alpha}(m^{-\alpha-2\beta}\|f\|_{\mathcal{F}}+1));\\
&\text{(b)} \quad\left|\psi(\Bx; W^{(0)} + W_t) - g(\Bx; W^{(0)} + W_t)\right|\\
&  \leq O(\eta^3 t^3 m^{1-4\alpha}(m^{-\alpha-2\beta}\|f\|_{\mathcal{F}}+1)^3+\eta t m^{1-2\alpha-2\beta} (m^{-\alpha-2\beta}\|f\|_{\mathcal{F}}+1));\\
&\text{(c)} \quad\left\| \nabla_W \mathcal{L}(\psi(\Bx; W^{(0)} + W_t)) - \nabla_W \mathcal{L}(g(\Bx; W^{(0)} + W_t)) \right\|_{2,1} \\
&\leq O\Big(\eta^5 t^5  m^{2-7\alpha} (m^{-\alpha-2\beta}\|f\|_{\mathcal{F}}+1)^5 +\eta^3 t^3 m^{2-5\alpha-2\beta} (m^{-\alpha-2\beta}\|f\|_{\mathcal{F}}+1)^3\\
&+\eta^2 t^2  m^{2-4\alpha-3\beta}  (m^{-\alpha-2\beta}\|f\|_{\mathcal{F}}+1)^2
+\eta^2 t^2  m^{1-4\alpha-2\beta} \|f\|_{\mathcal{F}} (m^{-\alpha-2\beta}\|f\|_{\mathcal{F}}+1)^2 \\
&+\eta t  m^{2-3\alpha-4\beta} (m^{-\alpha-2\beta}\|f\|_{\mathcal{F}}+1) +m^{2-2\alpha-5\beta} + m^{1-2\alpha-4\beta} \|f\|_{\mathcal{F}} \Big).
\end{align*}

\end{theorem}
\begin{proof}
 (a) Denote the $(i,j)$-th entry of $W$ by $w_{ij}$. Using \eqref{06} and the assumption \eqref{15}, for every $i\in[m]$,
\begin{multline}\label{11}
 \left| \frac{\partial{ \psi\left(\Bx; W^{(0)}+W_t\right)}}{\partial{w_{ij}}} \right| \leq  \left|6d a_i^{(0) } x_j (  {\Bw_i^{(t)}}^\top \Bx + b_i^{(0)} )^2 \cdot \mathbb{I}_{{\Bw_i^{(t)}}^\top \Bx + b_i^{(0)} \geq 0 }\right|\\
 +\left|12 a_i^{(0) } x_j({\Bw_i^{(t)}}^\top\Bx + b_i^{(0)} )^2   \cdot \mathbb{I}_{{\Bw_i^{(t)}}^\top\Bx + b_i^{(0)} \geq 0 }\right|\\
 +\left|6 a_i^{(0) } x_j({\Bw_i^{(t)}}^\top \Bw_i^{(t)} )(\|\Bx\|^2_2-1)\cdot \mathbb{I}_{{\Bw_i^{(t)}}^\top\Bx + b_i^{(0)} \geq 0 }\right|\\
 +\left|24 a_i^{(0) } x_j( {\Bw_i^{(t)}}^\top\Bx + b_i^{(0)} ) ({\Bw_i^{(t)}}^\top\Bx) \cdot \mathbb{I}_{{\Bw_i^{(t)}}^\top\Bx + b_i^{(0)} \geq 0 }\right|\\
+\left|12a_i^{(0)} w_{ij}^{(t)}( {\Bw_i^{(t)}}^\top\Bx + b_i^{(0)} )  (\|\Bx\|^2_2-1) \cdot \mathbb{I}_{{\Bw_i^{(t)}}^\top\Bx + b_i^{(0)} \geq 0 }\right|
\sim O(1)| a_i^{(0)}| .
\end{multline}
Consider the gradient $\nabla_{\Bw_i} \psi(\Bx; W^{(0)} + W_t) = \left[
\frac{\partial \psi}{\partial w_{i1}},\dots,\frac{\partial \psi}{\partial w_{id}}\right]$ ,
whose 2-norm is given by
\begin{equation} 
\begin{split} \label{21}
  \left\| \nabla_{\Bw_i} \psi(\Bx; W^{(0)} + W_t) \right\|_2 &= \left( \left( \frac{\partial \psi}{\partial w_{i1}} \right)^2 + \left( \frac{\partial \psi}{\partial w_{i2}} \right)^2 + \cdots + \left( \frac{\partial \psi}{\partial w_{id}} \right)^2 \right)^{\frac{1}{2}}\\
  &\leq \sqrt{d} \cdot {  O(1)} |a_i^{(0)} | \leq  O( m^{-\alpha})
\end{split}
\end{equation}
since $a_i^{(0)}$ is initialized by \eqref{07}.
From \eqref{08} and \eqref{16}, we have $\|\Bzeta(\Bx;\theta)\|_2\leq O(m^{-\alpha-2\beta})$ and $\|\Balpha(\theta)\|_2\leq \|f\|_{\mathcal{F}} \cdot p(\theta) $. Then
\begin{multline}
    \label{20}
    |f(\Bx)| =\left| \int_{\Lambda} \Balpha(\theta)^\top \Bzeta(\Bx; \theta)   \td\theta \right|\leq  \int_{\Lambda}|\Balpha(\theta)^\top \Bzeta(\Bx; \theta)| \td\theta\\ \leq \int_{\Lambda} \|\Balpha(\theta)\|_2 \|\Bzeta(\Bx;\theta)\|_2\td \theta
    \leq O(m^{-\alpha-2\beta} \|f\|_{\mathcal{F}})\cdot\int_{\Lambda} p(\theta) \td \theta =O(m^{-\alpha-2\beta} \|f\|_{\mathcal{F}}).
\end{multline}
Note that the gradient of the loss function 
\begin{multline} 
    \nabla_{\Bw_i} \mathcal{L}(\psi(\Bx; W^{(0)} + W_t)) = \nabla_{\Bw_i} \left( f(\Bx) - \psi(\Bx; W^{(0)} + W_t) \right)^2 \\
= 2 (\psi(\Bx;W^{(0)} + W_t)-f(\Bx))\cdot \nabla_{\Bw_i}(\psi(\Bx; W^{(0)} + W_t)).
\end{multline} 
Using \eqref{21}, \eqref{20} and assumption \eqref{15} leads to
    \begin{equation} \label{35}
    \| \nabla_{\Bw_i} \mathcal{L}(\psi(\Bx; W^{(0)} + W_t)) \|_2 \leq O(m^{-\alpha}(m^{-\alpha-2\beta}\|f\|_{\mathcal{F}}+1)).
    \end{equation} 
In the iterative framework of the SGD algorithm, we update the weights $w^{(t)}_i$ according to:
\begin{equation}
\begin{split}
\Bw^{(1)}_i = \Bw^{(0)}_i &- \eta \nabla_{\Bw_i} \mathcal{L}(\psi(\Bx; W^{(0)})),\\
\Bw^{(2)}_i = \Bw^{(1)}_i &- \eta \nabla_{\Bw_i} \mathcal{L}(\psi(\Bx; W^{(1)})), \\
&\cdots \\
\Bw^{(t)}_i = \Bw^{(t-1)}_i &- \eta \nabla_{\Bw_i} \mathcal{L}(\psi(\Bx; W^{(t-1)})).\nonumber\end{split}
\end{equation}
Based on this, the difference between the updated weights and initial weights can be bounded as 
\begin{equation}
    \label{38}
\|\Bw^{(t)}_i - \Bw^{(0)}_i \|_2 \leq \eta \sum_{k=0}^{t-1} \| \nabla_{\Bw_i} \mathcal{L}(\psi(\Bx; W^{(k)}) \|_2 = O(\eta t m^{-\alpha}(m^{-\alpha-2\beta}\|f\|_{\mathcal{F}}+1)).
\end{equation} 
Then, we complete the proof.

(b) By \eqref{07}, we know that
$|b_i^{(0)}|\leq m^{-\beta}$ and $ \|\Bw_i^{(0)} \|_2 \leq d^{1/2}m^{-\beta} $. So, from \eqref{38}, we can obtain
\begin{multline} \label{37}
\|\Bw_i^{(t)}\|_2\leq O(\eta t m^{-\alpha}(m^{-\alpha-2\beta}\|f\|_{\mathcal{F}}+1))+ \|\Bw_i^{(0)}\|_2\\
\leq O(\eta t m^{-\alpha}(m^{-\alpha-2\beta}\|f\|_{\mathcal{F}}+1)+ m^{-\beta}),
\end{multline}
then, we derive that
\begin{equation}
\label{22}
| {\Bw_i^{(t)}}^\top\Bx+b_i^{(0)}| \leq \| \Bw_i^{(t)}\|_2\| \Bx\|_2 +|b_i^{(0)}|
\leq  O(\eta t m^{-\alpha}(m^{-\alpha-2\beta}\|f\|_{\mathcal{F}}+1)+ m^{-\beta}) 
\end{equation}   
and 
\begin{equation} \label{23}
    |{\Bw_i^{(0)}}^\top\Bx+b_i^{(0)}|\leq \| \Bw_i^{(0)}\|_2\| \Bx\|_2 +|b_i^{(0)}| \leq O(m^{-\beta}),
\end{equation}
for sufficiently large $m$.

We rewrite \eqref{06} and \eqref{18} as $\psi=\sum_{i=1}^m \psi_i $, $g=\sum_{i=1}^m g_i$, where
\begin{multline}
\psi_i(\Bx;W)=a_i^{(0)}\Big(2d({\Bw_i}^\top\Bx+b_i^{(0)})^3+12({\Bw_i}^\top\Bx)({\Bw_i}^\top\Bx+b_i^{(0)})^2\\
    +6({\Bw_i}^\top\Bx+b_i^{(0)})({\Bw_i}^\top\Bw_i)(\|\Bx\|_2^2-1)\Big)\cdot\mathbb{I}_{{\Bw_i}^\top\Bx+b_i^{(0)}\geq 0}
\end{multline}
and
\begin{multline}
g_i(\Bx;W)=a_i^{(0)}\Big(2d({\Bw_i}^\top\Bx+b_i^{(0)})({\Bw_i^{(0)}}^\top\Bx+b_i^{(0)})^2\\
+12({\Bw_i}^\top\Bx+b_i^{(0)})({\Bw_i^{(0)}}^\top\Bx)({\Bw_i^{(0)}}^\top\Bx+b_i^{(0)})\\
    +6({\Bw_i}^\top\Bx+b_i^{(0)})({\Bw_i^{(0)}}^\top\Bw_i^{(0)})(\|\Bx\|_2^2-1)\Big)\cdot\mathbb{I}_{{\Bw_i^{(0)}}^\top\Bx+b_i^{(0)}\geq 0}.
\end{multline}
Denote $I_i^{(0)}:={\Bw_i^{(0)}}^\top \Bx + b_i^{(0)}$ and $I_i^{(t)}:={\Bw_i^{(t)}}^\top \Bx + b_i^{(0)}$. Then
\begin{multline}\label{36}
\left| \psi_i \left( \Bx; W^{(0)} + W_t \right) - g_i \left( \Bx; W^{(0)} + W_t \right) \right| \\
\leq\;\bigg|\; 2 d a_i^{(0)}  I_i^{(t)} \cdot\left( (  I_i^{(t)} )^2 \cdot \mathbb{I} _{ I_i^{(t)} \geq 0}
- ( I_i^{(0)} )^2 \cdot \mathbb{I}_{ I_i^{(0)} \geq 0} \right)\;\bigg|\; \\
+ \;\bigg|\; 12 a_i^{(0)} I_i^{(t)} \cdot
 \left( I_i^{(t)}  ({\Bw_i^{(t)}}^\top\Bx)\cdot\mathbb{I}_{  I_i^{(t)} \geq 0 }-  I_i^{(0)} ({\Bw_i^{(0)}}^\top\Bx)\cdot\mathbb{I}_{ I_i^{(0)} \geq 0}\right) \;\bigg|\;\\
+\;\bigg|\;6 a_i^{(0)} I_i^{(t)}\cdot(\|\Bx\|^2_2-1)\left(( {\Bw_i^{(t)}}^\top\Bw_i^{(t)})\cdot \mathbb{I} _{ I_i^{(t)} \geq 0}-({\Bw_i^{(0)}}^\top\Bw_i^{(0)})\cdot\mathbb{I}_{ I_i^{(0)} \geq 0}\right)\;\bigg|\;.
\end{multline}
We use the mean value theorem on $H_1(\Bw):=(\Bw^\top \Bx+b_i^{(0)})^2 \cdot\mathbb{I}_{ \Bw^\top \Bx+b_i^{(0)}\geq 0}$, then by \eqref{38}, \eqref{22} it follows that
\begin{multline} \label{39}
\left| (  I_i^{(t)} )^2 \cdot \mathbb{I} _{ I_i^{(t)} \geq 0}
- ( I_i^{(0)} )^2 \cdot \mathbb{I}_{ I_i^{(0)} \geq 0} \right|=\left|H_1(\Bw_i^{(t)})-H_1(\Bw_i^{(0)})\right|\\
= \left|\nabla H_1(\tilde{\Bw_i})^\top(\Bw_i^{(t)}-\Bw_i^{(0)})\right| \leq \|\nabla H_1(\tilde{\Bw_i})\|_2\|\Bw_i^{(t)}-\Bw_i^{(0)}\|_2\\
\leq 2|({\tilde{\Bw_i}}^\top\Bx+b_i^{(0)})|\cdot\|\Bx\|_2 \|\Bw_i^{(t)}-\Bw_i^{(0)}\|_2\\\leq  O(\eta t m^{-\alpha}(m^{-\alpha-2\beta}\|f\|_{\mathcal{F}}+1)+ m^{-\beta})\cdot  O(\eta t m^{-\alpha}(m^{-\alpha-2\beta}\|f\|_{\mathcal{F}}+1)),
\end{multline}
where $\tilde{\Bw_i}$ is some vector in $\mathbb{R}^d$ satisfying $\|\tilde{\Bw_i}\|_2\leq\max\{\|\Bw_i^{(t)}\|_2,\|\Bw_i^{(0)}\|_2\}$.

Similarly, defining $H_2(\Bw)=(\Bw^\top \Bx+b_i^{(0)})(\Bw^\top \Bw)\cdot \mathbb{I}_{\Bw^\top \Bx+b_i^{(0)}\geq0}$ and $H_3(\Bw)=(\Bw^\top \Bw)\cdot \mathbb{I}_{\Bw^\top \Bx+b_i^{(0)}\geq0}$ and using the mean value theorem lead to the same upper bound for 
$$\left|I_i^{(t)}  ({\Bw_i^{(t)}}^\top\Bx)\cdot\mathbb{I}_{  I_i^{(t)} \geq 0 }-  I_i^{(0)} ({\Bw_i^{(0)}}^\top\Bx)\cdot\mathbb{I}_{ I_i^{(0)} \geq 0}\right |$$ 
and 
$$\left|({\Bw_i^{(t)}}^\top\Bw_i^{(t)})\cdot \mathbb{I} _{ I_i^{(t)} \geq 0}-({\Bw_i^{(0)}}^\top\Bw_i^{(0)})\cdot\mathbb{I}_{ I_i^{(0)} \geq 0}\right|.$$
Therefore, \eqref{36} is bounded above by
\begin{multline*}
   (2d m^{-\alpha}+12m^{-\alpha}+12  m^{-\alpha} )\cdot  O(\eta t m^{-\alpha}(m^{-\alpha-2\beta}\|f\|_{\mathcal{F}}+1)+ m^{-\beta})^2 \cdot\  O(\eta t m^{-\alpha}(m^{-\alpha-2\beta}\|f\|_{\mathcal{F}}+1)\\ \sim  O(\eta^3 t^3 m^{-4\alpha}(m^{-\alpha-2\beta}\|f\|_{\mathcal{F}}+1)^3+\eta t m^{-2\alpha-2\beta} (m^{-\alpha-2\beta}\|f\|_{\mathcal{F}}+1)).
\end{multline*} 
Thus, the overall error satisfies
\begin{multline}
       \left| \psi \left( \Bx; W^{(0)} + W_t \right) - g \left( \Bx; W^{(0)} + W_t \right) \right|  \\
       \leq m \cdot  O(\eta^3 t^3 m^{-4\alpha}(m^{-\alpha-2\beta}\|f\|_{\mathcal{F}}+1)^3+\eta t m^{-2\alpha-2\beta} (m^{-\alpha-2\beta}\|f\|_{\mathcal{F}}+1))\\
       =  O(\eta^3 t^3 m^{1-4\alpha}(m^{-\alpha-2\beta}\|f\|_{\mathcal{F}}+1)^3+\eta t m^{1-2\alpha-2\beta} (m^{-\alpha-2\beta}\|f\|_{\mathcal{F}}+1)).   
\end{multline}

(c) Using the above results, we give sharper estimates of the upper bounds of $\| \nabla_{\Bw_i} \psi(\Bx; W^{(0)} + W_t) \|_2$
. By \eqref{37} and  \eqref{22}, we return to \eqref{11} to refine the upper bound for $\left| \frac{\partial{ \psi\left(\Bx; W^{(0)}+W_t\right)}}{\partial{w_{ij}}} \right|$, which is given by

\begin{equation}
\left| \frac{\partial{ \psi\left(\Bx; W^{(0)}+W_t\right)}}{\partial{w_{ij}}} \right|\leq O(\eta^2 t^2  m^{-3\alpha}(m^{-\alpha-2\beta}\|f\|_{\mathcal{F}}+1)^2+ m^{-\alpha-2\beta}). 
\end{equation}
Thus, we can bound the 2-norm of the gradient as
\begin{multline} \label{24}
    \left\| \nabla_{\Bw_i} \psi(\Bx; W^{(0)} + W_t) \right\|_2 = \left( \sum_{j=1}^d \left(\frac{\partial \psi}{\partial w_{ij}}\right)^2 \right)^{\frac{1}{2}} \leq \sqrt{d} \cdot O(\eta^2 t^2  m^{-3\alpha}(m^{-\alpha-2\beta}\|f\|_{\mathcal{F}}+1)^2+ m^{-\alpha-2\beta})\\ =  O(\eta^2 t^2  m^{-3\alpha}(m^{-\alpha-2\beta}\|f\|_{\mathcal{F}}+1)^2+ m^{-\alpha-2\beta}).
\end{multline}
For $ \left\| \nabla_{\Bw_i} \psi(\Bx; W^{(0)} + W_t) \right\|_2$, by \eqref{23} we have 
\begin{multline*}
 \left| \frac{\partial{ g(\Bx; W^{(0)}+W_t)}}{\partial{w_{ij}}} \right| 
 \leq\left|12 a_i^{(0)}x_j (  {\Bw_i^{(0)}}^\top\Bx) ({\Bw_i^{(0)}}^\top\Bx +b_i^{(0)})\cdot \mathbb{I}_{ {\Bw_i^{(0)}}^\top\Bx + b_i^{(0)} \geq 0 }\right|\\
 + \left|2da_i^{(0)} x_j (  {\Bw_i^{(0)}}^\top\Bx + b_i^{(0)} )^2 \cdot \mathbb{I}_{ {\Bw_i^{(0)}}^\top\Bx + b_i^{(0)} \geq 0 }\right|\\
+\left|6 a_i^{(0)}x_j({\Bw_i^{(0)}}^\top\Bw_i^{(0)})  (\|\Bx\|^2_2-1) \cdot\mathbb{I}_{ {\Bw_i^{(0)}}^\top\Bx + b_i^{(0)} \geq 0 }\right| \sim O( m^{-\alpha-2\beta}),   
\end{multline*}
and
\begin{equation}
    \label{25}
      \left\| \nabla_{\Bw_i} g(\Bx; W^{(0)} + W_t) \right\|_2 = \left( \sum_{j=1}^d \left(\frac{\partial g}{\partial w_{ij}}\right)^2 \right)^{\frac{1}{2}} \leq  O( m^{-\alpha-2\beta}).
\end{equation}
For $\left|\psi(\Bx;  W^{(0)} + W_t)\right|$, we recall its expression from \eqref{06}, whose upper bound is given by
\begin{multline}\label{27}
\left|\psi(\Bx;  W^{(0)} + W_t)\right| =\;\bigg|\; 2d\sum_{i=1}^{m} a_i^{(0)} (   {\Bw_i^{(t)}}^\top\Bx + b_i^{(0)} )^3\mathbb{I}_{  {\Bw_i^{(t)}}^\top \Bx + b_i^{(0)} \geq 0 }\\
+12\sum_{i=1}^{m} a_i^{(0)} (   {\Bw_i^{(t)}}^\top\Bx + b_i^{(0)} )^2(\Bw_i^{(t)\top } \Bx )\mathbb{I}_{  {\Bw_i^{(t)}}^\top \Bx + b_i^{(0)} \geq 0 }\\
+6\sum_{i=1}^{m} a_i^{(0)} ( {\Bw_i^{(t)}}^\top \Bx + b_i^{(0)})( {\Bw_i^{(t)}}^\top \Bw_i^{(t)}) (\|\Bx\|^2_2-1)\mathbb{I}_{  {\Bw_i^{(t)}}^\top \Bx + b_i^{(0)} \geq 0 }\;\bigg|\; \\
\leq O(\eta ^3 t^3 m^{1-4\alpha} (m^{-\alpha-2\beta}\|f\|_{\mathcal{F}}+1)^3+ m^{1-\alpha-3\beta}). 
\end{multline} 
Similarly, from \eqref{18} we have
\begin{multline}\label{28}
\left|g(\Bx;  W^{(0)} + W_t)\right| =\;\bigg|\; 2d\sum_{i=1}^{m} a_i^{(0)} (  {\Bw_i^{(0)}}^\top\Bx + b_i^{(0)} )^2 (   {\Bw_i^{(t)}}^\top\Bx + b_i^{(0)} )\mathbb{I}_{ {\Bw_i^{(0)}}^\top\Bx + b_i^{(0)} \geq 0 }\\
+12\sum_{i=1}^{m} a_i^{(0)} (  {\Bw_i^{(0)}}^\top\Bx + b_i^{(0)} ) (   {\Bw_i^{(t)}}^\top\Bx + b_i^{(0)} )({\Bw_i^{(0)}}^\top\Bx )\mathbb{I}_{ {\Bw_i^{(0)}}^\top\Bx + b_i^{(0)} \geq 0 }\\
+6\sum_{i=1}^{m} a_i^{(0)} (  {\Bw_i^{(t)}}^\top\Bx + b_i^{(0)} )({\Bw_i^{(0)}}^\top\Bw_i^{(0)})(\|\Bx\|^2_2-1)\mathbb{I}_{ {\Bw_i^{(0)}}^\top\Bx + b_i^{(0)} \geq 0 }\;\bigg|\;\\ \leq O(\eta t  m^{1-2\alpha-2\beta} (m^{-\alpha-2\beta}\|f\|_{\mathcal{F}}+1)+ m^{1-\alpha-3\beta}) .
\end{multline} 

Combining \eqref{20}  and \eqref{24}-\eqref{28},  we have
\begin{multline}
\left\| \nabla_W \mathcal{L}(\psi(\Bx; W^{(0)} + W_t)) - \nabla_W\mathcal{L}(g(\Bx; W^{(0)} + W_t)) \right\|_{2,1} \\ 
=\sum_{i\in [m]}\left\|\nabla_{\Bw_i}\left(\psi(\Bx;W^{(0)}+W_t)-f(\Bx)\right)^2-\nabla_{\Bw_i}\left(g(\Bx;W^{(0)}+W_t)-f(\Bx)\right)^2\right\|_2\\
\leq\sum_{i\in [m]} 2\left(|\psi(\Bx;W^{(0)}+W_t)|+|f(\Bx)|\right)\cdot\|\nabla_{\Bw_i}\psi(\Bx;W^{(0)}+W_t)\|_2\\
+\sum_{i\in [m]} 2\left(|g(\Bx;W^{(0)}+W_t)|+|f(\Bx)|\right)\cdot\|\nabla_{\Bw_i}g(\Bx;W^{(0)}+W_t)\|_2\\
\leq m \cdot \Big[  O(\eta ^3 t^3 m^{1-4\alpha} (m^{-\alpha-2\beta}\|f\|_{\mathcal{F}}+1)^3+ m^{1-\alpha-3\beta}+ m^{-\alpha-2\beta}\|f\|_{\mathcal{F}})\\\cdot  O(\eta^2 t^2  m^{-3\alpha}(m^{-\alpha-2\beta}\|f\|_{\mathcal{F}}+1)^2+ m^{-\alpha-2\beta})\Big]\\
+ m \cdot \Big[O(\eta t  m^{1-2\alpha-2\beta} (m^{-\alpha-2\beta}\|f\|_{\mathcal{F}}+1)
\cdot O( m^{-\alpha-2\beta})\Big]\\
\leq O\Big(\eta^5 t^5  m^{2-7\alpha} (m^{-\alpha-2\beta}\|f\|_{\mathcal{F}}+1)^5 +\eta^3 t^3 m^{2-5\alpha-2\beta} (m^{-\alpha-2\beta}\|f\|_{\mathcal{F}}+1)^3\\
+\eta^2 t^2  m^{2-4\alpha-3\beta}  (m^{-\alpha-2\beta}\|f\|_{\mathcal{F}}+1)^2
+\eta^2 t^2  m^{1-4\alpha-2\beta} \|f\|_{\mathcal{F}} (m^{-\alpha-2\beta}\|f\|_{\mathcal{F}}+1)^2 \\+\eta t  m^{2-3\alpha-4\beta} (m^{-\alpha-2\beta}\|f\|_{\mathcal{F}}+1) +m^{2-2\alpha-5\beta} + m^{1-2\alpha-4\beta} \|f\|_{\mathcal{F}} \Big) .
\end{multline}  
\end{proof}

\subsection{Main result of optimization}
Now, we present the main theorem of the optimization analysis.
\begin{theorem} \label{thm3.4}
Suppose $f\in\mathcal{F}$ and $\alpha+3\beta>1$. For any $\epsilon\in(0,1]$ and $\delta>0$, let
\begin{multline*}
    M=\max\Bigg\{ \Bigg(\frac{(2C_d\|f\|_{\mathcal{F}}(1+\sqrt{2\log\frac{1}{\delta}}))^2}{\epsilon} \Bigg)^{\frac{1}{2\alpha+4\beta+1}},\left(\frac{C_d'}{\epsilon}\right)^{\frac{1}{\alpha+3\beta-1}},\\\left(\frac{\|f\|_{\mathcal{F}}}{\epsilon}\right)^{\frac{1}{2\alpha+5\beta-1}},\left(\frac{\|f\|^2_{\mathcal{F}}}{\epsilon}\right)^{\frac{1}{2\alpha+4\beta}}\Bigg\},
\end{multline*}
where $C_d$ and $C_d'$  defined in Lemma \ref{lem04} and lemma \ref{lem 05}, respectively. Let 
\begin{multline}
T_0=C_f\min\Bigg\{\frac{m^{\frac{1+3\alpha+\beta}{2}}}{\epsilon^{\frac{3}{4}}},\frac{m^{\frac{1+5\alpha+3\beta}{3}}}{\epsilon^{\frac{2}{3}}},\frac{m^{\frac{2+4\alpha}{3}}}{\epsilon^{\frac{2}{3}}},\frac{m^{2\alpha+2\beta}}{\epsilon^{\frac{1}{2}}},\\\frac{m^{-1+3\alpha+5\beta}}{\epsilon},
\frac{m^{\frac{2+5\alpha+2\beta}{3}}}{\epsilon^{\frac{2}{3}}},\frac{m^{\frac{1+4\alpha+3\beta}{2}}}{\epsilon^{\frac{1}{2}}},\frac{m^{1+2\alpha+\beta}}{\epsilon^{\frac{1}{2}}}
\Bigg\}
\end{multline} with 
$C_f:=\frac{1}{(\|f\|_\mathcal{F}+1)^2\max\{\|f\|_\mathcal{F},1\}}$ only depending on $f$. If $m$ is sufficiently large such that $m\geq M$ and $T_0>\frac{\|f\|^2_\mathcal{F}}{\epsilon^2}$, then with number of iterations $T\in\left[\frac{\|f\|^2_\mathcal{F}}{\epsilon^2},T_0\right]$ and learning rate $\eta=\Theta\left(\frac{\epsilon}{m}\right)$, with probability at least $1-\delta$ over the random initialization, the average loss after $T$ iterations of SGD satisfies
\begin{equation}\label{13}
\mE_{X}\mE_{sgd}\left[\frac{1}{T}\sum_{t=0}^{T-1}\mathcal{L}_{\Psi}(X;W_t)\right] \leq O(\epsilon),
\end{equation}
where $\mE_{X}$ takes the expectation over the random choice of data set $X$ under distribution $\mathcal{D}$ and $\mE_{sgd}$ takes the expectation over the random choice of the training points $\Bx\sim\mathcal{U}(X)$ in the SGD algorithm.
\end{theorem}
\begin{proof}
First, denote
\begin{gather} 
\mathcal{L}_{\Psi}(\Bx;W):=\mathcal{L}({\psi}(\Bx;W^{(0)}+W)),\\
\mathcal{L}_G(\Bx;W):=\mathcal{L}(g(\Bx;W^{(0)}+W)).
\end{gather}
For the set of samples $X$, we denote the empirical losses by
\begin{gather} 
\mathcal{L}_{\Psi}(X;W):= \frac{1}{N}\sum_{\Bx\in X }\mathcal {L}({\psi}(\Bx;W^{(0)}+W)),\\
\mathcal{L}_G(X;W):= \frac{1}{N}\sum_{\Bx\in X }\mathcal{L}(g(\Bx;W^{(0)}+W)).
\end{gather}
For two matrices $A=(a_{ij})_{p\times q}$ and $B=(b_{ij})_{p\times q}$, we define their inner product by $\left<A,B\right>=\sum_{i=1}^p\sum_{j=1}^q a_{ij}b_{ij}$. 

From Corollary 3.2, with probability at least $1-\delta$ over the random initialization, there exists $W^*=[\Bw_1^*~\dots~\Bw_m^*]$ with $\|W^*\|_{2,\infty}\leq \frac{\|f\|_{\mathcal{F}}}{m}$ and $\|W^*\|_F\leq \frac{\|f\|_{\mathcal{F}}}{\sqrt{m}}$. Recall that $\mathcal{L}$ is convex and $g(\Bx;W)$ is linear in $W$, so $\mathcal{L}_G$ is convex in $W$. Applying the mean value theorem, we have
\begin{multline}\label{eq83}
 \mathcal{L}_G(X;W_t) -\mathcal{L}_G(X;W^{*}) \leq\left< \nabla_W\mathcal{L}_G(X;W_t), W_t-W^*\right>\\
 =\left< \nabla_W\mathcal{L}_G(X;W_t)-\nabla_W\mathcal{L}_{\Psi}(X;W_t)+\nabla_W\mathcal{L}_{\Psi}(X;W_t), W_t-W^*\right>\\
 =\left< \nabla_W\mathcal{L}_G(X;W_t)-\nabla_W\mathcal{L}_{\Psi}(X;W_t), W_t-W^*\right>+\left<\nabla_W\mathcal{L}_{\Psi}(X;W_t), W_t-W^*\right>\\
 \leq \|\nabla_W\mathcal{L}_G(X;W_t)-\nabla_W\mathcal{L}_{\Psi}(X;W_t)\|_{2,1} \|W_t-W^*\|_{2,\infty}+\left<\nabla_W\mathcal{L}_{\Psi}(X;W_t), W_t-W^*\right>.
\end{multline}

From the SGD algorithm, we also have
\begin{multline}
\|W_{t+1}-W^* \|_F^2 =\|W_t- \eta \nabla_W \mathcal{L}_{\Psi}(\Bx^{(t)};W_t)-W^*\|_F^2 \\
=\|W_{t}-W^*\|_F^2-2\eta \left<\nabla_W\mathcal{L}_{\Psi}(\Bx^{(t)};W_t), W_t-W^* \right>+ \eta^2\|\nabla_W\mathcal{L}_{\Psi}(\Bx^{(t)};W_t)\|_F^2, 
\end{multline}
where $\Bx^{(t)}\sim\mathcal{U}(X)$ is the random sample of the $t$-th iteration. Next, we consider the inner product between $\nabla_W\mathcal{L}_{\Psi}(X;W_t)$ and $W_t-W^*$, i.e.,
\begin{multline}\label{eq85}
\left<\nabla_W \mathcal{L}_{\Psi}(X;W_t), W_t-W^* \right>=\frac{1}{N}\sum_{\Bx^{(t)}\in X}\left<\nabla_W\mathcal{L}_{\Psi}(\Bx^{(t)};W_t), W_t-W^* \right>\\
=\frac{1}{N}\sum_{\Bx^{(t)}\in X}\frac{\|W_{t}-W^*\|_F^2-\|W_{t+1}-W^* \|_F^2+\eta^2\|\nabla_W\mathcal{L}_{\Psi}(\Bx^{(t)};W_t)\|_F^2}{2\eta}\\     
=\frac{\|W_{t}-W^*\|_F^2-\frac{1}{N}\sum_{\Bx^{(t)}\in X}(\|W_{t+1}-W^* \|_F^2-\eta^2\|\nabla_W\mathcal{L}_{\Psi}(\Bx^{(t)};W_t)\|_F^2)}{2\eta}.
\end{multline}
Plugging \eqref{eq85} into \eqref{eq83} and taking expectation over the random choice of $\Bx^{(t)}$ leads to
\begin{multline}\label{eq86}
 \mathcal{L}_G(X;W_t) -\mathcal{L}_G(X;W^{*}) \leq \|\nabla_W\mathcal{L}_G(X;W_t)-\nabla_W\mathcal{L}_{\Psi}(X;W_t)\|_{2,1} \|W_t-W^*\|_{2,\infty}\\
+\frac{\|W_t-W^*\|_F^2-\mathbb{E}_{\Bx^{(t)}}\|W_{t+1}-W^*\|_F^2}{2 \eta}+\frac{\eta}{2}\mathbb{E}_{\Bx^{(t)}}\left[\|\nabla_W\mathcal{L}_{\Psi}(\Bx^{(t)};W_t)\|_F^2\right].  
\end{multline}
Recall $W_t=W^{(t)}-W^{(0)}$, writing $W^{(t)}=[\Bw_1^{(t)}-\Bw_1^{(0)},\dots, \Bw_m^{(t)}-\Bw_m^{(0)} ]$, then using Theorem \ref{thm3.3}(a) and the fact that $\|W^*\|_{2,\infty}\leq \frac{\|f\|_{\mathcal{F}}}{m}$, we have
\begin{equation}
    \label{eq87}
 \|W_t-W^*\|_{2,\infty} \leq \|W_t\|_{2,\infty}+ \|W^*\|_{2,\infty}
 \leq O\left(\eta t m^{-\alpha} (m^{-\alpha-2\beta}\|f\|_{\mathcal{F}}+1)+\frac{\|f\|_{\mathcal{F}}}{m}\right). 
\end{equation}
And, by \eqref{21},
\begin{equation}
 \label{eq88}
\|\nabla_W\mathcal{L}_{\Psi}(\Bx^{(t)};W_t)\|_F^2=\sum_{i\in [m]}\|\nabla_{\Bw_i}\mathcal{L}_{\Psi}(\Bx^{(t)}; W_t)\|_2^2 
\leq m\cdot O(( m^{-\alpha})^2)=O( m^{1-2\alpha}). 
\end{equation}
By Theorem 3.3(c), we have
\begin{multline}\label{eq90} 
\|\nabla_W\mathcal{L}_G(X;W_t)-\nabla_W\mathcal{L}_{\Psi}(X;W_t)\|_{2,1}\\
=\left\|\frac{1}{N}\sum_{\Bx\in X }\nabla_W\mathcal{L}(g(\Bx;W^{(0)}+W_t))-\frac{1}{N}\sum_{\Bx\in X }\nabla_W\mathcal{L}({\psi}(\Bx;W^{(0)}+W_t))\right\|_{2,1}\\
\leq O\Big(\eta^5 T^5  m^{2-7\alpha} (m^{-\alpha-2\beta}\|f\|_{\mathcal{F}}+1)^5 +\eta^3 T^3 m^{2-5\alpha-2\beta} (m^{-\alpha-2\beta}\|f\|_{\mathcal{F}}+1)^3\\
+\eta^2 T^2  m^{2-4\alpha-3\beta}  (m^{-\alpha-2\beta}\|f\|_{\mathcal{F}}+1)^2
+\eta^2 T^2  m^{1-4\alpha-2\beta} \|f\|_{\mathcal{F}} (m^{-\alpha-2\beta}\|f\|_{\mathcal{F}}+1)^2 \\+\eta T  m^{2-3\alpha-4\beta} (m^{-\alpha-2\beta}\|f\|_{\mathcal{F}}+1) +m^{2-2\alpha-5\beta} + m^{1-2\alpha-4\beta} \|f\|_{\mathcal{F}} \Big):=I'.     
\end{multline}
Therefore, averaging up \eqref{eq86} from $t=0$ to $T-1$ and combining with \eqref{eq87}, \eqref{eq88} and \eqref{eq90}, we obtain the following result for the average optimization error 
\begin{multline}\label{eq91}
\frac{1}{T}\sum_{t=0}^{T-1} \mE_{sgd}[\mathcal{L}_G(X;W_t)]-\mathcal{L}_G(X;W^*)\\
\leq O(m^{-\alpha} \eta T(m^{-\alpha-2\beta}\|f\|_{\mathcal{F}}+1)I'+\frac{\|f\|_{\mathcal{F}}}{m}I')+\frac{\|W_0-W^*\|_F^2}{2\eta T}+ O(\eta  m^{1-2\alpha}).
\end{multline}
Since $\eta=\Theta\left(\frac{\epsilon} {m}\right)$, the third term in \eqref{eq91} is bounded by
\begin{equation}
O(\eta m^{1-2\alpha})=O(\epsilon m^{-2\alpha})\leq O(\epsilon).
\end{equation}
Since $\|W_0-W^*\|_F^2=\|W^*\|_F^2 \leq O\left(\frac{\|f\|_{\mathcal{F}}^2}{m}\right)$, by the hypothesis that $T\geq\frac{\|f\|_{\mathcal{F}}^2}{\epsilon^2}$, the second term in \eqref{eq91} is bounded by
\begin{equation}
\frac{\|W_0-W^*\|_F^2}{2\eta T}\leq O(\epsilon).
\end{equation}
Next, using the hypothesis of $\eta$ and $T$ again, the first term in \eqref{eq91} is bounded by
\begin{multline}
O(m^{-\alpha} \eta T(m^{-\alpha-2\beta}\|f\|_{\mathcal{F}}+1)I'+\frac{\|f\|_{\mathcal{F}}}{m}I')\\
\leq  O\Big( T^6  \epsilon^6 m^{-4-8\alpha}(m^{-\alpha-2\beta}\|f\|_{\mathcal{F}}+1)^6+T^4 \epsilon^4 m^{-2-6\alpha-2\beta} (m^{-\alpha-2\beta}\|f\|_{\mathcal{F}}+1)^4 \\ + T^3 \epsilon^3 m^{-1-5\alpha-3\beta}  (m^{-\alpha-2\beta}\|f\|_{\mathcal{F}}+1)^3+T^3  \epsilon^3   m^{-2-4\alpha}(m^{-\alpha-2\beta}\|f\|_{\mathcal{F}}+1)^4   \\
+ T^2 \epsilon^2 m^{-4\alpha-4\beta} (m^{-\alpha-2\beta}\|f\|_{\mathcal{F}}+1)^2+ T \epsilon  m^{1-3\alpha-5\beta} (m^{-\alpha-2\beta}\|f\|_{\mathcal{F}}+1) \\+T \epsilon m^{-3\alpha-4\beta}\|f\|_{\mathcal{F}}  (\|f\|_{\mathcal{F}}+1) +T^5 \epsilon^5 m^{-4-7\alpha} \|f\|_{\mathcal{F}} (m^{-\alpha-2\beta}\|f\|_{\mathcal{F}}+1)^5\\
+T^3 \epsilon^3  m^{-2-5\alpha-2\beta} \|f\|_{\mathcal{F}}(m^{-\alpha-2\beta}\|f\|_{\mathcal{F}}+1)^3
+ T^2   \epsilon^2  m^{-1-4\alpha-3\beta}\|f\|_{\mathcal{F}}(m^{-\alpha-2\beta}\|f\|_{\mathcal{F}}+1)^2
\\+ T^2 \epsilon^2  m^{-2-4\alpha-2\beta}  \|f\|^2_{\mathcal{F}}(m^{-\alpha-2\beta}\|f\|_{\mathcal{F}}+1)^2
+ T \epsilon m^{-2\alpha-2\beta}  (m^{-\alpha-2\beta}\|f\|_{\mathcal{F}}+1)^2\\
+ m^{1-2\alpha-5\beta}\|f\|_{\mathcal{F}}+  m^{-2\alpha-4\beta} \|f\|_{\mathcal{F}}^2\Big)\leq O(\epsilon).
\end{multline} 
Therefore, from \eqref{eq91} we have that
\begin{equation}\label{4.60}
\frac{1}{T}\sum_{t=0}^{T=1} \mE_{sgd}[\mathcal{L}_G(X;W_t)]-\mathcal{L}_G(X;W^*)
\leq  O(\epsilon).
\end{equation}

By Theorem \ref{thm3.3}(b), \eqref{27}, \eqref{28} and \eqref{20} the difference between $\mathcal{L}_F(X;W_t)$ and $\mathcal{L}_G(X;W_t)$ is given by
\begin{multline}
        |\mathcal{L}_F(X;W_t)-\mathcal{L}_G(X;W_t)|\\
        \leq \frac{1}{N}\sum_{\Bx\in X }\left|\left(f(\Bx)-{\psi}(\Bx;W^{(0)}+W_t)\right)^2-\left(f(\Bx)-g(\Bx;W^{(0)}+W_t)\right)^2\right|\\
        \leq \frac{1}{N}\sum_{\Bx\in X }  \left|{\psi}(\Bx;W^{(0)}+W_t)-g(\Bx;W^{(0)}+W_t)\right|\\
        \cdot \left(|{\psi}(\Bx;W^{(0)}+W_t)|+|g(\Bx;W^{(0)}+W_t)|+|2f(\Bx)|\right)\\
        \leq   \frac{1}{N}\sum_{x\in X }  O(\eta^3 t^3 m^{1-4\alpha}(m^{-\alpha-2\beta}\|f\|_{\mathcal{F}}+1)^3+\eta t m^{1-2\alpha-2\beta} (m^{-\alpha-2\beta}\|f\|_{\mathcal{F}}+1)) \\
        \cdot \Big(  O(\eta ^3 t^3 m^{1-4\alpha} (m^{-\alpha-2\beta}\|f\|_{\mathcal{F}}+1)^3+ m^{1-\alpha-3\beta}) \\+ O(\eta t  m^{1-2\alpha-2\beta} (m^{-\alpha-2\beta}\|f\|_{\mathcal{F}}+1)+ m^{1-\alpha-3\beta})+ O( m^{-\alpha-2\beta}\|f\|_{\mathcal{F}})\Big)\\
       =O\left(\eta^3 t^3 m^{1-4\alpha}(m^{-\alpha-2\beta}\|f\|_{\mathcal{F}}+1)^3+\eta t m^{1-2\alpha-2\beta} (m^{-\alpha-2\beta}\|f\|_{\mathcal{F}}+1)\right)^2\\
        +O\Big(\eta^3 t^3 m^{2-5\alpha-3\beta}(m^{-\alpha-2\beta}\|f\|_{\mathcal{F}}+1)^3+\eta t m^{2-3\alpha-5\beta}(m^{-\alpha-2\beta}\|f\|_{\mathcal{F}}+1)\\+\eta^3 t^3 m^{1-5\alpha-2\beta}\|f\|_{\mathcal{F}}(m^{-\alpha-2\beta}\|f\|_{\mathcal{F}}+1)^3+\eta t m^{1-3\alpha-4\beta}\|f\|_{\mathcal{F}}(m^{-\alpha-2\beta}\|f\|_{\mathcal{F}}+1)\Big)\\
        \leq  O(\epsilon),
        \label{93}
\end{multline}
where $\eta=\Theta\left(\frac{\epsilon} {m}\right)$ and the relation $t\leq T$ are used. 

From Corollary 3.2, with probability at least $1-\delta$ over random initialization, we have $\mE_{\Bx\sim \mathcal{D}}[\mathcal{L}_G(\Bx;W^*)]\leq \epsilon$. Taking the expectation over the entire dataset, we have
\begin{equation}\label{94}
\mE_{X}[\mathcal{L}_G(X;W^*)]=\mE_X[\frac{1}{N}\sum_{\Bx\in X}\mathcal{L}_G(\Bx;W^*)]\leq \epsilon. 
\end{equation}
Then, plugging \eqref{94} into\eqref{4.60}, we obtain
\begin{equation}\label{95}
\mE_X[\frac{1}{T}\sum_{t=0}^{T=1} \mE_{sgd}[\mathcal{L}_G(X;W_t)]]\leq \epsilon+O(\epsilon) \sim O(\epsilon). 
\end{equation}
Finally, combining \eqref{93} with \eqref{95}, with probability at least $1-\delta$ over random initialization, we have the estimation \eqref{13}.
\end{proof}

Theorem \ref{thm3.4} implies that for sufficiently wide PINNs, SGD with appropriate iteration numbers and learning rates can decrease the average training loss below any given accuracy $O(\epsilon)$. Note that the theorem only requires the width  $m=\Omega(\frac{c_f}{\epsilon^p})$ for some $p>0$ and some $f$-dependent constant $c_f>0$; the requirement is independent of the number of training samples $N$. 

Moreover, the condition $\alpha+3\beta>1$ guarantees that the powers of $m$ in the definition of $T_0$ are all positive, making $T_0>\frac{\|f\|^2_\mathcal{F}}{\epsilon^2}$ possible when $m$ is sufficiently large. Note that the usual choice $(\alpha,\beta)=(0,1/2)$ is also included in this condition.

\section{Generalization Analysis} 
Now, we consider the generalization results of the PINN model. First, we have
\begin{theorem} \label{thm4.1}
Given $0\leq\tau' \leq 1$ and $N \geq 1$. Let $\Bx_n\in\mathbb{R}^d$ with $\|\Bx_n\|_2\leq1$ for $n=1,\dots, N$. Then the empirical Rademacher complexity of the function class $\mathcal{F}_\psi:=\{\Bx \mapsto\psi(\Bx;W^{(0)}+W') \mid\|W'\|_{2,\infty}\leq\tau'\}$ is bounded by
\begin{equation*} 
\widehat{\mathcal{R}}(X;\mathcal{F}_\psi
)=\frac{1}{N} \mE_{\xi \in \{\pm 1\}^N} \left[ \sup_{\|W'\|_{2, \infty} \leq \tau'} \sum_{n =1}^N\xi_n {\psi}(\Bx_n; W^{(0)} + W') \right]
\leq O\left( \frac{ m^{-\alpha} \tau'}{\sqrt{N}} \right),
\end{equation*}
where $\xi=[\xi_1,\dots,\xi_N]$ is the vector of Rademacher random variables, which is of uniform distribution. i.e., $\mathbb{P}(\xi_n=1)=\mathbb{P}(\xi_n=-1)=\frac{1}{2}$ for all $n$. 
\end{theorem}
\begin{proof}
We denote $X=\{\Bx_1,\dots ,\Bx_N\}$ as the set of samples and define the function class $\mathcal{F}_1=\{\Bx \mapsto (\Bw'_i)^\top\Bx \mid \|\Bw'_i\|_2 \leq \tau' \}$. According to Lemma \ref{lem 05}(a), the empirical Rademacher complexity  with respect to \(X\) of $\mathcal{F}_1$ is
\begin{equation}
         \widehat{\mathcal{R}}_1(X; \mathcal{F}_1)=\frac{1}{N} \mE_{\xi \in \{\pm 1\}^N} \left[ \sup_{\|\Bw'_i\|_2 \leq \tau'} \sum_{n =1}^N \xi_n ((\Bw'_i)^\top\Bx_n) \right]\leq O\left(\frac{\tau'}{\sqrt{N}}\right).
    \end{equation}

Similarly, we define $\mathcal{F}_2=\{\Bx \mapsto  (\Bw^{(0)}_i + \Bw'_i)^\top\Bx + b^{(0)}_i \mid \|\Bw'_i\|_2 \leq \tau'\}$. Since the singleton class has zero complexity and adding it does not affect complexity, applying Lemma \ref{lem 05}(b), we establish that
\begin{multline}
         \widehat{\mathcal{R}}_2(X; \mathcal{F}_2)=\frac{1}{N} \mE_{\xi \in \{\pm 1\}^N} \left[ \sup_{\|\Bw'_i\|_2\leq \tau'} \sum_{n =1}^N \xi_n (\Bw_i^{(0)}+\Bw'_i)^\top \Bx_n)+b_i^{(0)} )\right] \\   
         =\frac{1}{N} \mE_{\xi \in \{\pm 1\}^N} \left[ \sup_{\|\Bw'_i\|_2\leq \tau'} \sum_{n =1}^N \xi_n ({\Bw'_i}^\top \Bx_n)+b_i^{(0)} )\right]\\+\frac{1}{N} \mE_{\xi \in \{\pm 1\}^N} \left[ \sup_{\|\Bw^{(0)}_i\|_2\leq d^{1/2}m^{-\beta}} \sum_{n =1}^N \xi_n ({\Bw^{(0)}_i}^\top \Bx_n)\right]
         \leq O\left(\frac{\tau'}{\sqrt{N}}\right).
\end{multline}

Writing $W'=[\Bw_1'~\dots~\Bw_m']$, ${\psi}(\Bx;W^{(0)}+W')$ from \eqref{06} is given by
\begin{multline}
{\psi}(\Bx;W^{(0)}+W')=2d \cdot \sum_{i=1}^m a_i^{(0)}\sigma((\Bw_i^{(0)}+\Bw'_i)^\top \Bx+b_i^{(0)})\\
+4\cdot  \sum_{i=1}^m a_i^{(0)}\sigma'((\Bw_i^{(0)}+\Bw'_i)^\top \Bx+b_i^{(0)}) \cdot (\Bw_i^{(0)}+\Bw'_i)^\top \Bx\\
+(\|\Bx\|^2_2-1) \sum_{i=1}^m a_i^{(0)}\sigma''((\Bw_i^{(0)}+\Bw'_i)^\top \Bx+b_i^{(0)})\cdot (\Bw_i^{(0)}+\Bw'_i)^\top(\Bw_i^{(0)}+\Bw'_i)\\
=2d \cdot 3\cdot \sum_{i=1}^m a_i^{(0)}\cdot\gamma_{1}((\Bw_i^{(0)}+\Bw'_i)^\top \Bx+b_i^{(0)})\\
        +4 \cdot 6 \cdot  \sum_{i=1}^m a_i^{(0)}(\Bw_i^{(0)}+\Bw'_i)^\top \Bx \cdot \gamma_2 ((\Bw_i^{(0)}+\Bw'_i)^\top \Bx+b_i^{(0)})  \\
        +6\cdot (\|\Bx\|^2_2-1) \sum_{i=1}^m a_i^{(0)} (\Bw_i^{(0)}+\Bw'_i)^\top (\Bw_i^{(0)}+\Bw'_i)\cdot \gamma_3((\Bw_i^{(0)}+\Bw'_i)^\top \Bx+b_i^{(0)}) 
\end{multline}
where $\gamma_{1}(\Bx)=\frac{1}{3}\max (0,\Bx)^3$ , $ \gamma_{2}(\Bx)=\frac{1}{2}\max (0,\Bx)^2 $ and $\gamma_{3}(\Bx)=\max (0,\Bx)$.  
We denote 
\begin{align*}
J^{(1)}_n&=6d \sum_{i=1}^m a_i^{(0)}\cdot\gamma_{1}((\Bw_i^{(0)}+\Bw'_i)^\top \Bx_n+b_i^{(0)}),\\
J^{(2)}_{n}&=24\sum_{i=1}^m a_i^{(0)}(\Bw_i^{(0)}+\Bw'_i)^\top \Bx_n \cdot \gamma_2 ((\Bw_i^{(0)}+\Bw'_i)^\top \Bx_n+b_i^{(0)}),\\
J^{(3)}_{n}&=6 (\|\Bx_n\|^2_2-1) \sum_{i=1}^m a_i^{(0)} (\Bw_i^{(0)}+\Bw'_i)^\top (\Bw_i^{(0)}+\Bw'_i)\cdot \gamma_3((\Bw_i^{(0)}+\Bw'_i)^\top \Bx_n+b_i^{(0)}).
\end{align*}
Then
\begin{multline*}
    \widehat{\mathcal{R}}(X; \mathcal{F}_\psi)=\frac{1}{N} \mE_{\xi \in \{\pm 1\}^N} \left[ \sup_{\|W'\|_{2, \infty} \leq \tau'} \sum_{n =1}^N \xi_n \psi(\Bx_n; W^{(0)} + W') \right]\\
    =\frac{1}{N} \mE_{\xi \in \{\pm 1\}^N} \left[ \sup_{\|W'\|_{2, \infty} \leq \tau'} \sum_{n=1} ^N\xi_n \left( J^{(1)}_{n}+J^{(2)}_{n}+J^{(3)}_{n}\right)\right],
\end{multline*}
where  $\mathcal{F}_\psi=\{\Bx \mapsto \psi(\Bx; W^{(0)} + W') \mid \|W'\|_{2,\infty} \leq \tau'\}$,  $ \widehat{\mathcal{R}}(X;\mathcal{F}_\psi)$ presents the empirical Rademacher complexity with respect to \(X\) of \(\mathcal{F}_\psi\). Denote $\Ba^{(0)}=[a_1^{(0)},\dots , a_m^{(0)}]$. Since $\gamma_1,\gamma_2,\gamma_3$ are 1-Lipschitz continuous, by Lemma \ref{lem 05}(c) we have
   \begin{equation*}
      \frac{1}{N} \mE_{\xi \in \{\pm 1\}^N} \left[ \sup_{\|W'\|_{2, \infty} \leq \tau'} \sum_{ n =1}^N \xi_n J_n^{(1)}\right]
      \leq  2\left\|2d\cdot 3\cdot a^{(0)}\right\|_1 \cdot \widehat{\mathcal{R}}_2 (X; \mathcal{F}_2)
      \leq O( m^{-\alpha})\cdot O\left(\frac{\tau'}{\sqrt{N}}\right)
       = O\left(\frac{ m^{-\alpha} \tau'}{\sqrt{N}}\right),
        \end{equation*}
\begin{multline*}
      \frac{1}{N} \mE_{\xi \in \{\pm 1\}^N} \left[ \sup_{\|W'\|_{2, \infty} \leq \tau'} \sum_{ n =1}^N \xi_n J_n^{(2)}\right]
      \leq  2 \left\| 4\cdot 6 \cdot \Ba^{(0)} \cdot \max_{1\leq i\leq m}\left\{(\Bw_i^{(0)}+\Bw'_i)^\top \Bx\right\}\right\|_1\cdot \widehat{\mathcal{R}}_2(X; \mathcal{F}_2)\\
      \leq  O\left((\tau'+ d^{1/2}m^{-\beta})\cdot  m^{-\alpha}\right) \cdot O\left(\frac{\tau'}{\sqrt{N}}\right)
       = O\left(\frac{ m^{-\alpha} (\tau')^2}{\sqrt{N}}\right)+O\left(\frac{ m^{-\alpha-\beta} \tau'}{\sqrt{N}}\right)
       \end{multline*}
       and
\begin{multline*}
      \frac{1}{N} \mE_{\xi \in \{\pm 1\}^N} \left[ \sup_{\|W'\|_{2, \infty} \leq \tau'} \sum_{ n =1}^N \xi_n J_n^{(3)}\right]\\
      \leq  2 \left\|6\cdot(\|\Bx\|^2_2-1)\cdot a^{(0)} \cdot\max_{1\leq i\leq m}\left\{(\Bw_i^{(0)}+\Bw'_i)^\top (\Bw_i^{(0)}+\Bw'_i)\right\}  \right\|_1 \cdot \widehat{\mathcal{R}}_2(X; \mathcal{F}_2)\\
      \leq  O((\tau'+ d^{1/2}m^{-\beta})^2\cdot m^{-\alpha})\cdot O\left(\frac{\tau'}{\sqrt{N}}\right)
       = O\left(\frac{ m^{-\alpha} (\tau')^3}{\sqrt{N}}\right)+O\left(\frac{ m^{-\alpha-2\beta} \tau'}{\sqrt{N}}\right) .
       \end{multline*}
Finally, we have $\widehat{\mathcal{R}}(X;\mathcal{F}_\psi
)\leq O\left(\frac{ m^{-\alpha} \tau'}{\sqrt{N}}\right) $  since $0\leq\tau' \leq 1$.
\end{proof}

Next, we show the main theorem of generalization, which implies that the expected risk can also be decreased by SGD.
\begin{theorem}\label{thm4.2}
Under the hypothesis of Theorem \ref{thm3.4}, if 
$$N\geq N_0:=\frac{(m^{-\alpha-2\beta}\|f\|_{\mathcal{F}}+1)^2}{\epsilon^2}\max \left\{\log{(1/\delta),\eta^2 T^2 m^{-4\alpha}}\right\},$$ 
with probability at least $1-2\delta$, the average expected risk after $T$ iterations of SGD satisfies
\begin{equation}\label{46}
\mE_X\mE_{sgd}\left[\frac{1}{T}\sum_{t=0}^{T-1}\mE_{\Bx \sim \mathcal{D}} \mathcal{L}(\psi(\Bx;W^{(0)}+W_t))\right]\leq O(\epsilon).
\end{equation}
\end{theorem}

\begin{proof}
From Theorem 3.4, with probability at least $1-\delta$ over random initialization, the training loss satisfies
\begin{equation}\label{44}
\mE_{X}\mE_{sgd}\left[\frac{1}{T}\sum_{t=0}^{T-1}\mathcal{L}_{\Psi}(X;W_t)\right] \leq O(\epsilon).
\end{equation} 
Recall $W_t= W^{(t)}-W^{(0)}$ and Theorem \ref{thm3.3}(a), we can bound
\begin{equation}
    \|W_t\|_{2,\infty}=\max_{1\leq i \leq m}\|\Bw_i^{(t)}-\Bw_i^{(0)}\|_2 \leq O(\eta T m^{-\alpha}(m^{-\alpha-2\beta}\|f\|_{\mathcal{F}}+1)).
\end{equation}
Then, let $\tau'=O(\eta T m^{-\alpha}(m^{-\alpha-2\beta}\|f\|_{\mathcal{F}}+1))$. From Theorem 4.1, the empirical Rademacher complexity
\begin{multline}
\widehat{\mathcal{R}}(X;\mathcal{F}_\psi)=\mE_{\xi \in \{\pm 1\}^N} \left[ \frac{1}{N} \sup_{\|W'\|_{2, \infty} \leq \tau'} \sum_{n=1}^N \xi_n \psi(\Bx_n; W^{(0)} + W') \right] \\
\leq O\left( \frac{ m^{-\alpha} \tau'}{\sqrt{N}} \right)=O\left( \frac{\eta T m^{-2\alpha}(m^{-\alpha-2\beta}\|f\|_{\mathcal{F}}+1)}{\sqrt{N}} \right).
\end{multline}
Note that the loss function $\mathcal{L}$ defined in \eqref{48} is continuous with Lipschitz constant $O(m^{-\alpha-2\beta}\|f\|_{\mathcal{F}}+1)$. By Corollary \ref{cor02} with $C=O(m^{-\alpha-2\beta}\|f\|_{\mathcal{F}}+1)$, with probability at least $1-\delta$ over the randomness of $X$, we have
\begin{multline}\label{42}
\left|\mE_{\Bx \sim \mathcal{D}} \mathcal{L}(\psi(\Bx;W^{(0)}+W_t))-\mathcal{L}_\Psi(X;W_t)\right|\\
=\Bigg | \mE_{\Bx \sim \mathcal{D}} \left[ \left( f(\Bx) - \psi(\Bx;W^{(0)}+W_t ) \right)^2 \right] - \frac{1}{N} \sum_{n=1}^N \left( f(\Bx_n) - \psi(\Bx_n; W^{(0)}+W_t) \right)^2 \Bigg | \\
\leq O(\widehat{\mathcal{R}}(X; \mathcal{F}_\psi)) + O \left( \frac{(m^{-\alpha-2\beta}\|f\|_{\mathcal{F}}+1)\sqrt{\log \frac{1}{\delta}}}{\sqrt{N}} \right).
\end{multline} 

Using Theorem \ref{thm4.1} with $\tau'=O(\eta T m^{-2\alpha}(m^{-\alpha-2\beta}\|f\|_{\mathcal{F}}+1))$, it follows that
\begin{multline}
\left|\mE_{\Bx \sim \mathcal{D}} \mathcal{L}(\psi(\Bx;W^{(0)}+W_t))-\mathcal{L}_\Psi(X;W_t)\right|\\
\leq O\left(\frac{(m^{-\alpha-2\beta}\|f\|_{\mathcal{F}}+1)(\eta T m^{-2\alpha} +\sqrt{\log \frac{1}{\delta}}) }{\sqrt{N} }
\right)\leq O(\epsilon)
\end{multline}
since $N\geq N_0$. So,
\begin{equation}\label{45}
\frac{1}{T}\sum_{t=0}^{T-1}\left|\mE_{\Bx \sim \mathcal{D}} \mathcal{L}(\psi(\Bx;W^{(0)}+W_t))-\mathcal{L}_\Psi(X;W_t)\right| \leq O(\epsilon).
\end{equation} 
Combining \eqref{44} and \eqref{45}, it holds with probability at least $1-2\delta$ that
\begin{equation}
\mE_X\mE_{sgd}\left[\frac{1}{T}\sum_{t=0}^{T-1}\mE_{\Bx \sim \mathcal{D}} \mathcal{L}(\psi(\Bx;W^{(0)}+W_t))\right]\leq O(\epsilon).
\end{equation}
\end{proof}

Theorem \ref{thm4.2} demonstrates that under the hypothesis of Theorem \ref{thm3.4}, SGD has good generalization with average expected risk below $O(\epsilon)$ if the training data size is larger than $N_0$. Note that $N_0$ does not increase as $m$ increases; it has an upper bound independent of $m$.

\section{Numerical Experiments}
In this section, our theory is validated by the numerical results of solving Poisson's equation \eqref{01} with $d=3$ and $f(\Bx)=x_1^2+x_2^2+x_3^2$. We implement the SGD algorithm described in Section \ref{sec_SGD} with $T=10^6$ iterations. The training dataset $X$ is generated with uniform distribution on $\Gamma$. The learner network $\psi$ is set as \eqref{06} and initialized as \eqref{07} with $\alpha=0$ and $\beta=1/2$. We test using different choices of the network width $m$ and the number of training samples $N$. The average training loss and expected loss, i.e., 
\begin{equation*}
\frac{1}{T'}\sum_{t=0}^{T'-1}\frac{1}{N}\sum_{\Bx\in X }\mathcal {L}({\psi}(\Bx;W^{(t)}))~\text{and}~\frac{1}{T'}\sum_{t=0}^{T'-1}\mE_{\Bx \sim \mathcal{D}} \mathcal{L}(\psi(\Bx;W^{(t)})),
\end{equation*}
for $T'=1,\dots,T$ are computed (the expectation is approximately estimated on $10^5$ testing points in $\Gamma$). 

We plot the curves of average training losses versus iterations in Figure \ref{Fig01}. It is observed that the average training losses continue to decrease to a level between $10^{-3}$ and $10^{-4}$, which corresponds to the target accuracy $O(\epsilon)$ as predicted by the theory.

Also, we list the average training losses and expected losses after $T$ iterations in Table \ref{Tab01}. First, for every width $m$, the training loss keeps the same magnitude for $N=100,1000,10000$, which implies that even if the number of samples increases significantly, the width required for the training loss to reach a certain value remains the same; namely, the width requirement is independent of the training data size. This is consistent with Theorem \ref{thm3.4}. Moreover, for every $m$, the expected loss decreases and gets closer to the training loss as $N$ increases, which means that achieving a small expected loss requires a sufficient number of training samples. This result is partially reflected by Theorem \ref{thm4.2} that sufficient $N$ is required for SGD to generalize well.

\begin{figure}
\centering
\subfloat[$N=100$]{
\includegraphics[scale=0.35]{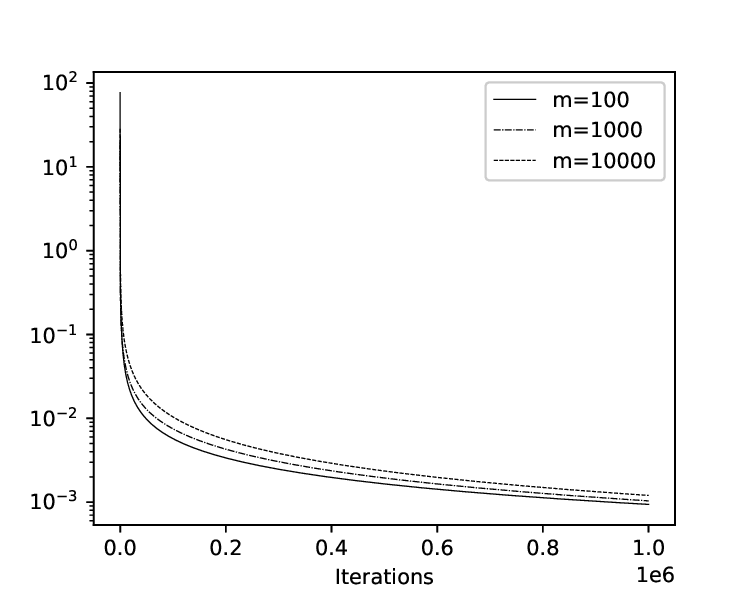}}
\subfloat[$N=1000$]{
\includegraphics[scale=0.35]{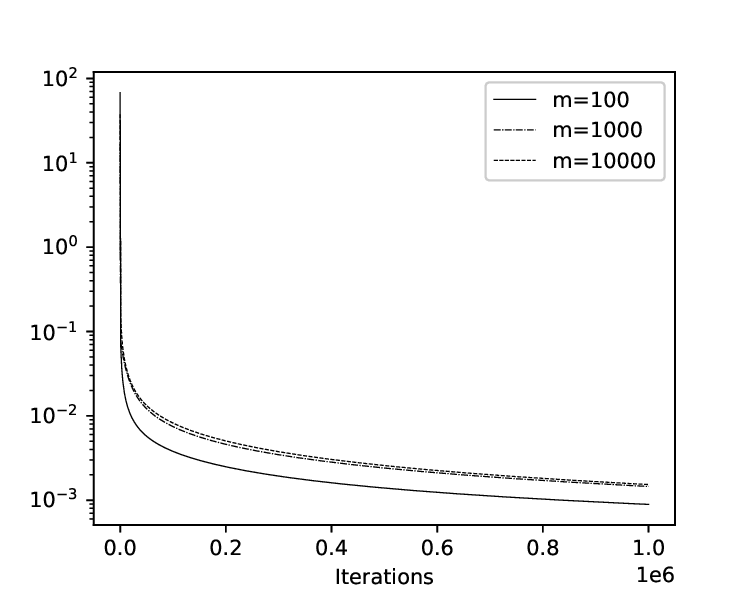}}
\subfloat[$N=10000$]{
\includegraphics[scale=0.35]{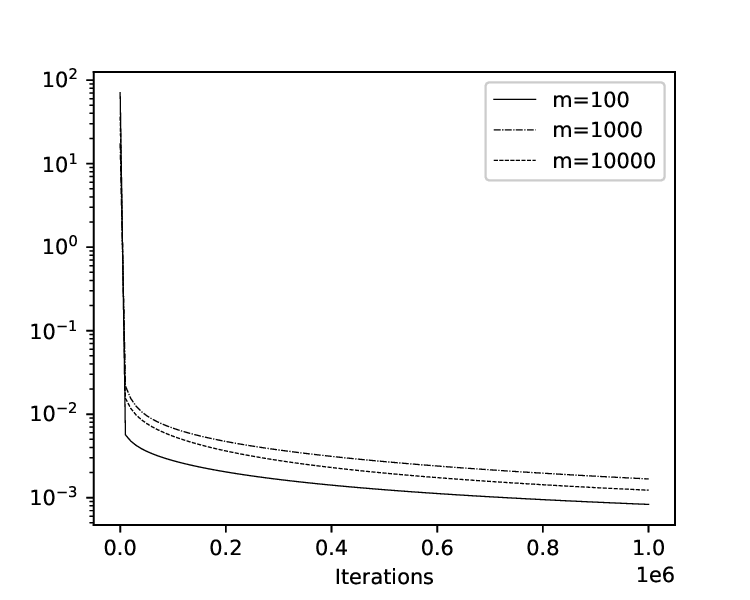}}
\caption{The training loss curves for various $m$ and $N$ during $T=10^6$ iterations.}
\label{Fig01}
\end{figure}

\begin{table}
\centering
\begin{tabular}{l|c|c|c}
& $m=100$ & $m=1000$ & $m=10000$\\\hline
$N=100$ & 9.40e-04 / 2.31e-03 & 1.03e-03 / 3.91e-03 & 1.20e-03 / 5.53e-03 \\\hline
$N=1000$ & 8.91e-04 / 9.99e-04 & 1.45e-03 / 1.72e-03 & 1.53e-03 / 1.71e-03 \\\hline
$N=10000$ & 8.31e-04 / 8.60e-04 & 1.68e-03 / 1.70e-03 & 1.23e-03 / 1.24e-03 \\
\end{tabular}
\caption{The final average training loss/average expected loss for various $m$ and $N$ after $T=10^6$ iterations.}
\label{Tab01}
\end{table}

\section{Conclusion}
This work establishes theoretical guarantees for successfully training two-layer PINNs using SGD. We construct a function class for the target function, i.e., the governing function of the PDE. After that, we analyze the optimization dynamics of SGD, obtaining the bounds for the average training loss. Specifically, we prove that the training loss can be decreased below $O(\epsilon)$ if the network width is larger than $\frac{c}{\epsilon^p}$ for some $p>0$ and some problem-dependent constant $c>0$, which is independent of the training data size; namely, the result does not require the over-parametrization hypothesis. A Similar result for the expected risk is also derived using Rademacher complexity. While we conduct the analysis on the PINNs associated with Poisson's equation, the framework can be easily extended to other types of PDEs.

One limitation of this paper lies in that only shallow PINNs with one hidden layer are considered. However, training deep neural networks may be essentially different from shallow ones since the weights of the outer layers and inner layers have distinct gradient representations. Future work could consider the behavior of gradient descent in training slightly deeper (e.g., three-layer) networks.

Another limitation lies in that we consider the PDE on the simple domain, i.e., the unit ball, which allows a simple network architecture of PINN that satisfies the boundary condition automatically. This simplifies the PINN model using a one-term loss for every training sample. However, such simplification is not always available for general domains, and the loss could contain two or more coupled terms, making analysis difficult. So, future work could also be studying PDEs on general domains with types of boundary or initial conditions.

\appendix 

\section{Technical Proofs for Lemmas}
\subsection{Proof of Lemma \ref{lem03}}

We use $\xi$ to denote every random variable in $\Xi$. For any $i\in\{1,\dots,m\}$, we let $\widetilde{\xi_n}$ be a random variable i.i.d. with $\xi_n$. Also, let $\widetilde{\Xi} = \{\xi_1, \cdots, \widetilde{\xi_n}, \cdots, \xi_m\}$ be a copy of $\Xi$ with the $i$-th element replaced by $\widetilde{\xi_n}$, and let $\bar{\widetilde{\Xi}}=$ be its average. Denote $E_\xi:=\mE\xi=\mE\overline{\Xi}=\mE\bar{\widetilde{\Xi}}$ and $h(\Xi):=\|\overline{\Xi}-\mE\overline{\Xi}\|$. Applying the triangle inequality give  
\begin{equation}
| h(\Xi) - h(\widetilde{\Xi}) | = \left|\|\overline{\Xi}-E_\xi\| - \|\bar{\widetilde{\Xi}}-E_\xi\|\right| \leq \| \overline{\Xi} - \bar{\widetilde{\Xi}} \| = \frac{\|\xi_n - \widetilde{\xi_n}\|}{m} 
\leq \frac{\|\xi_n \|+\|\widetilde{\xi_n}\|}{m} \leq \frac{2C}{m}.
\end{equation}

Next, we use the variance identity for the mean of i.i.d. random variables $\overline{\Xi}$ 
\begin{equation} 
\tVar(\overline{\Xi}) = \tVar(\frac{1}{m}\sum_{i=1}^{m}\xi_n)=\frac{1}{m}\tVar(\xi),
\end{equation}
which leads to
\begin{multline}\label{10}
\mE \left\|\overline{\Xi}-\mE\overline{\Xi}\right\|^2=\frac{1}{m}\mE\|\xi-\mE\xi\|^2=\frac{1}{m} \mE\langle\xi-E_\xi,\xi-E_\xi\rangle\\
=\frac{1}{m}(\mE\|\xi\|^2-2\mE\langle\xi,E_\xi\rangle+\|E_\xi\|^2)=\frac{1}{m}(\mE\|\xi\|^2-2\langle\mE\xi,E_\xi\rangle+\|E_\xi\|^2)\\
=\frac{1}{m}\left(\mE\|\xi\|^2-\|E_\xi\|^2\right)\leq\frac{1}{m}\mE\|\xi\|^2\leq\frac{C^2}{m}.
\end{multline}
Using Lemma \ref{lem01} Jensen’s inequality on the $\nu(t)=t^2$, we obtain
\begin{equation} (\mE h(\Xi))^2 \leq \mE(h^2(\Xi)),\end{equation}
so by \eqref{10},
\begin{equation}
\mE h(\Xi) \leq \sqrt{ \mE(h^2 (\Xi)) } = \sqrt{ \mE \left\| \overline{\Xi} - \mE \overline{\Xi} \right\|^2 } \leq \frac{C}{\sqrt{m}}.
\end{equation}

Then, applying Lemma \ref{lem02} McDiarmid’s inequality with this bound, we have
\begin{equation}
\mathbb{P} \left[ h(\Xi) - \frac{C}{\sqrt{m}} \geq \epsilon \right] \leq \mathbb{P} \left[ h(\Xi) - \mE h(\Xi) \geq \epsilon \right] \leq \exp \left( -\frac{m \epsilon^2}{2C^2} \right),
\end{equation} 
letting $\epsilon=\sqrt{\frac{2C^2 \log{(1/\delta)}}{m}}$ leads to the result.
\subsection{Proof of Lemma \ref{lem04}}

Let $f(\Bx) = \int_{\Lambda} \Balpha(\theta)^\top \Bzeta(\Bx; \theta) \td\theta$ for some $\Balpha$ that achieves 
$\|f\|_{\mathcal{F}}=\max_{\theta\in \Lambda} \frac{\|\Balpha(\theta)\|_2}{p(\theta)}$. For $i = 1,\ldots,m$, we construct $\Balpha_i=\frac{\Balpha(\theta_i)}{mp(\theta_i)}$, then $\|\Balpha_i\|_2\leq\frac{\|f\|_{\mathcal{F}}}{m}$ and $g(\Bx):= \sum_{i=1}^{m} \Balpha_i^\top\Bzeta(\Bx;\theta_i)\in\mathcal{F}_{m}$. We also have  
\begin{multline}
        \mE_{\theta_1,\dots,\theta_m} [g(\Bx)] =\frac{1}{m} \sum_{i=1}^{m} \mE_{\theta_i}\left[\frac{\Balpha(\theta_i)^\top \Bzeta(\Bx;\theta_i)}{p(\theta_i)}\right]\\
        =\frac{1}{m}\sum_{i=1}^{m}\int_{\Lambda}\Balpha(\theta_i)^\top \Bzeta(\Bx;\theta_i)\td\theta_i=\int_{\Lambda} \Balpha(\theta)^\top \Bzeta(\Bx;\theta)\td\theta=f(\Bx).
\end{multline}
Note that $\Bzeta(\Bx;\theta)$ is a vector-valued function. We use $\Bzeta_{ij}(\Bx)$ to denote the $j$-th component of $\Bzeta(\Bx;\theta_i)$, and denote $\theta_i=\left(a_i^{(0)},\Bw_i^{(0)},b_i^{(0)}\right)$. By the expression \eqref{08} of $\Bzeta$ and the fact that $\|\Bx\|_2\leq1$, we have
\begin{multline}\label{30}
\max_k|\Bzeta_{ik}| \leq |a_i^{(0)}| \cdot    \Bigg (
2d\left(\|\Bw_i^{(0)}\|_2 \|\Bx\|_2+|b_i^{(0)}|\right)^2
\\+ 12\left(\|\Bw_i^{(0)}\|_2\|\Bx\|_2 \right) 
\left(\|\Bw_i^{(0)}\|_2\|\Bx\|_2+|b_i^{(0)}| \right)  
+ 6\|\Bw_i^{(0)}\|_2^2\left(\|\Bx\|_2^2 + 1\right)\Bigg)\\
\leq m^{-\alpha}\left(2d\left(d^{1/2} m^{-\beta} + m^{-\beta} \right)^2+12 d^{1/2} m^{-\beta}\left(d^{1/2} m^{-\beta} + m^{-\beta}\right)+12 d m^{-2\beta}\right)\\
= d^{-1/2}C_dm^{-\alpha-2\beta} ,
\end{multline} 
which leads to
\begin{equation}\label{29}
    \begin{split}
    \|\Bzeta(\Bx;\theta_i)\|_2=(\Bzeta_{i1}(\Bx)^2+\dots+\Bzeta_{id}(\Bx)^2)^{\frac{1}{2}}\leq C_dm^{-\alpha-2\beta} ,
    \end{split}
\end{equation}
for all $\Bx\in\Gamma$. Consider the Hilbert space $L^2_\mu(\Gamma)$ which contains functions from $\Gamma$ to $\mathbb{R}$ associated with inner product 
\begin{equation}
\langle \tilde{f}, \tilde{g}\rangle:=\int_\Gamma\tilde{f}(\Bx)\tilde{g}(\Bx)\td\mu(\Bx),\quad\forall\tilde{f},\tilde{g}\in L^2_\mu(\Gamma).
\end{equation}
By \eqref{29},
\begin{multline}
\|\Balpha_i^\top\Bzeta(\Bx;\theta_i)\|_{L^2_\mu(\Gamma)}=\sqrt{\int_\Gamma|\Balpha_i^\top\Bzeta(\Bx;\theta_i)|^2\td \mu(\Bx)}\\
\leq\sqrt{\int_\Gamma\|\Balpha_i\|_2^2\|\Bzeta(\Bx;\theta_i)\|_2^2\td \mu(\Bx)}
\leq C_dm^{-\alpha-2\beta} \|\Balpha_i\|_2 \leq C_d\|f\|_\mathcal{F}m^{-\alpha-2\beta-1}.
\end{multline}
Note that $g(\Bx)= \frac{1}{m}\sum_{i=1}^{m} m\Balpha_i^\top\Bzeta(\Bx;\theta_i)$. The proof is completed by applying Lemma \ref{lem03} to $\left\{m\Balpha_i^\top\Bzeta(\Bx;\theta_i)\right\}_{i\in[m]}$ in the Hilbert space $L^2_\mu(\Gamma)$.

\section*{Funding}
This work is supported by National Natural Science Foundation of China Major Research Plan (G0592370101).

\bibliographystyle{plainnat}
\bibliography{references}
\end{document}